\documentclass{article}

\usepackage[main, final]{neurips_2026}

\usepackage[utf8]{inputenc}
\usepackage[T1]{fontenc}

\usepackage{notations}
\usepackage{subcaption}  
\usepackage{stmaryrd}

\usepackage[colorlinks=false]{hyperref}
\usepackage[capitalize,noabbrev]{cleveref}

\usepackage{tikz}
\usetikzlibrary{
    decorations.pathmorphing,
    positioning,
    decorations.shapes,
    shapes.geometric,  
    arrows.meta,      
    fit,             
}
\usepackage{placeins}
\usepackage{xcolor}

\usepackage[most]{tcolorbox}
\newtcolorbox{conclusionbox}{
  colback=blue!5,        
  colframe=blue!75!black, 
  coltitle=black,        
  fonttitle=\bfseries,   
  boxrule=1pt,           
  arc=1mm,               
  left=2mm,              
  right=2mm,             
  top=1mm,               
  bottom=1mm,            
}

\usepackage[textsize=tiny]{todonotes}

\title{When to Align, When to Predict: A Phase Diagram for Multimodal Learning}
\author{%
  Ilay Kamai\thanks{Corresponding author.} \\
  Technion \\
  \texttt{ilay.kamai@campus.technion.ac.il} \\
  \And
  Hugues Van Assel \\
  Genentech \\
  \And
  Aviv Regev \\
  Genentech \\
  \And
  Hagai B.\ Perets \\
  Technion \\
  \And
  Randall Balestriero \\
  Brown University\\
}
\date{}

\begin{document}

\maketitle
\begin{abstract}
Cross-modal alignment (CA) and cross-modal prediction (CP) are the dominant paradigms for
multimodal representation learning, yet there is no systematic understanding of when each
succeeds and when each fails, a gap that leaves practitioners, especially in scientific
domains with heterogeneous instruments and multiple levels of measurement, unable to
diagnose why standard methods underperform the best single modality. We study both
objectives under a spiked signal-plus-noise model with structured cross-modal nuisance
correlation, the ingredient that breaks the classical recovery guarantees, and derive
separation ratios that expose complementary failure modes: alignment whitens each modality
and fails when nuisance is strongly correlated across views; prediction encodes whatever is
cross-predictable through a one-sided whitening, with recovery governed by source-modality
quality. The resulting phase diagram partitions multimodal problems into four regimes ---
Both, CA only, CP only, and Neither --- refined by a recovery count that separates partial
recovery from complete failure. We present a data-driven procedure to locate real-world
datasets in this diagram using a small labeled subsample, identifying the preferred
objective and prediction direction before any cross-modal training, and identifying when no
objective in the CA/CP family can improve on the stronger modality alone. Experiments on
synthetic data, stereo-vision benchmarks, image--caption pairs, and two real scientific
domains, astronomy and single-cell multi-omics, validate the predictions in the
nonlinear regime, including both faces of the Neither regime. Code to reproduce the results
is available at \url{https://github.com/IlayMalinyak/mm_align_vs_pred}.
\end{abstract}
\section{Introduction}\label{sec:intro}
Multimodal representation learning aims to extract a shared latent structure from paired observations across different modalities, such as images and captions, audio and video, molecular cell profiles and tissue images, or different telescopes observing the same object. Multimodal learning is crucial when a single modality alone is insufficient to fully describe a phenomenon of interest or when the information in a single modality is degenerate or noisy. Moreover, combining multiple modalities of the same object is an important building block for foundational models. Multimodal learning has achieved many successes across domains and scales (e.g., ~\cite{Cui2025_biology, Parker2025_aion, Alayrac2022_flamingo, Bodnar2024_aurora}). However, the field is mostly empirical, and theoretical studies are
relatively sparse, though phenomena like the modality gap in contrastive multimodal models \cite{Liang2022_mind_the_gap} have started attracting principled analysis (e.g., \cite{Yossef_Levi2024_clip_geomtry}). 

Here, we focus on the interplay between the two leading multimodal learning paradigms - \emph{Cross-modal alignment} and \emph{Cross-modal prediction}.
\emph{Cross-modal alignment} (CA) projects paired samples into a common embedding space, encouraging matched pairs to be close; CLIP~\cite{radford2021_clip}, ImageBind~\cite{Girdhar2023_imagebind}, and VICReg~\cite{Bardes2021_vicreg} are prominent examples.
\emph{Cross-modal prediction} (CP) reconstructs one modality from the other through a bottleneck, so that the learned representation retains whatever is useful for prediction; masked autoencoders~\cite{he2021_MAE}, data2vec~\cite{Baevski2022_data2vec}, and the decoder side of encoder-decoder models follow this approach.
Both paradigms are widely used, yet they are typically studied in isolation and selected by practitioners based on empirical performance or architectural convenience rather than a principled understanding of their relative strengths and suitability to the problem at hand.
We shed light on fundamental characteristics in multimodal learning, when implemented with CA or CP, and provide practical guidelines for success and failure modes of the two. We derive exact solutions and recovery conditions in the linear case, verify the results for the non-linear case using various experiments with deep neural networks, and provide a data-driven method for the analysis of multimodality problems. 
To the best of our knowledge, this is the first work to systematically compare CA and CP under a multimodal spiked model with structured cross-modal nuisance correlation, and to translate the resulting recovery conditions into a practical diagnostic procedure that applies to real paired datasets.
Our main contributions are:
\begin{itemize}
\item \textbf{Unified linear analysis of CA and CP.}
Using the known equivalences of CA with Canonical Correlation Analysis (CCA) and CP with truncated reduced-rank regression (RRR) as a starting point, we derive closed-form solutions for both objectives and analyze them under a spiked signal-plus-noise model with cross-modal nuisance correlation. We derive separation ratios $\Delta_{\mathrm{CA}}$ and
$\Delta_{\mathrm{CP}}$ that determine when each method recovers the shared signal subspace, and expose complementary failure modes.
\item \textbf{Phase diagram with four recovery regimes.} The separation ratios
partition the space of multimodal problems into four regions — CA only, CP only,
Both, and Neither — visualized as a phase diagram in signal-noise space. We identify
the Neither regime as the natural habitat of complementary scientific modalities and
an important open problem for multimodal representation learning.
\item \textbf{Data-driven recovery regime estimation.} We propose an algorithm that
predicts the separation ratios for any paired dataset, based on a small labeled subsample,
and before any cross-modal training. Beyond this regime prediction, the per-modality noise
estimates identify which modality is stronger and in which direction CP should be
applied, a non-trivial question to address in practice. 
\item \textbf{Experimental validation across scales.} Experiments on synthetic
data, controlled stereo-vision benchmarks, and image--caption pairs confirm that
the failure modes identified in the linear theory persist with deep networks.
On real scientific data, spanning astronomy and single-cell multi-omics, we confirm the
predicted regime shift and observe both faces of the Neither regime: complete failure and
partial recovery.
\end{itemize}

The paper is organized as follows: in \Cref{sec:solutions_and_bottlenecks}, we present the methods and construct a multimodal spiked model with both modality-specific and cross-modal correlated noise features for the linear case. We then derive signal recovery conditions and phase diagrams. In \Cref{sec:experiments}, we provide experimental results that support the theory and an algorithm for estimating recovery regimes. Conclusions are provided in \Cref{sec:conclusion}.

\section{Related Work}
\paragraph{Theory of multimodal learning.}
A small but growing body of work studies multimodal learning theoretically. 
One study \cite{Huang2021_whatmakes} analyzed when using more modalities reduces population risk, showing through generalization bounds that the benefit depends on the gap in representation quality between modality subsets. 
Their linear analysis assumes orthonormal projections and full-rank covariance. 
Another work \cite{Lu2023_theory} proves that multimodal learning can achieve lower sample complexity than unimodal learning by decoupling the complexity of the connection function from the predictor. 
However, their analysis does not capture the effects of modality-specific noise, nuisance features, or rank deficiency that arise in practice. 
Both works establish that multimodality \emph{can} help under specific assumptions; our work characterizes when it works and when it \emph{can also hurt}, and shows that the specific approach (\textit{i.e.} CA or CP) used to combine modalities matters.
In a complementary direction, \cite{BetterTogether2025} studies spike detection in a multimodal spiked covariance model, comparing self-covariance, cross-covariance, and joint-covariance decompositions under finite-sample (Wishart) noise. 
They derive Baik--Ben~Arous--P\'ech\'e (BBP) phase transitions for each matrix and produce phase diagrams showing which method detects the signal first as a function of signal strength and sampling ratio.
Concurrently, \cite{Mergny2025_PLS} establish BBP-type thresholds for
partial least squares (PLS) and CCA under finite-sample Wishart noise, and \cite{Tabanelli2025}
extend this line to a multimodal spiked matrix--tensor model, showing that
joint maximum-likelihood optimization is strictly worse than a sequential
strategy. Both perform finite-sample BBP-type analyses of a single signal
model. Our contribution is complementary, comparing CA against CP at the
population level under structured cross-modal nuisance covariance $\eta_j$
in a matrix--matrix setting. This parameterization is what produces phenomena that prior spiked-model analyses did not capture, like source--target asymmetry of CP, and the \emph{Neither} regime in which neither paradigm fully recovers the signal.

\paragraph{Theory of self-supervised and contrastive learning.}
Our work also connects to a rich literature on the theory of unimodal self-supervised learning (SSL), where ``views'' are generated by data augmentation rather than being structurally distinct modalities.  
Several works derive closed-form solutions for linear SSL models and analyze the role of augmentations in shaping learned representations (e.g., ~\cite{Cabannes2023_ssl_interplay, Balestriero2022_spectral}).
Most closely related is \cite{VanAssel2025}, which compares joint-embedding and reconstruction-based SSL in a unified linear framework, the unimodal counterpart of our analysis. 
They show that joint-embedding methods impose a strictly weaker alignment condition on augmentations than reconstruction methods when irrelevant features have large magnitude, providing provable guidelines for choosing between the two paradigms. 
Our work extends this line of inquiry to the genuinely multimodal setting, where the two ``views'' are not designed augmentations of the same input but fixed, structurally distinct modalities with inherent asymmetries in quality and noise. 
This introduces new phenomena, such as modality bottlenecks, cross-modal nuisance correlation, and asymmetry of cross-prediction, that do not arise in the unimodal SSL framework. 
Other theoretical works on contrastive and non-contrastive learning (e.g., ~\cite{Arora2019_contrastive, HaoChen2021_spectral, Tian2021_understanding}) study downstream task performance as a function of augmentation design, but do not compare between alignment and prediction or address multimodal noise structure.

Beyond theory, our analysis also speaks to a core architectural choice: whether to add a predictor head on top of a joint embedding like in JEPA (\cite{lecun2022path}) architectures (\textit{e.g., } \cite{Assran2023_IJEPA, Balestriero2025_LeJEPA}) or align embeddings directly 
(\textit{e.g., } SimSiam \cite{Chen2020_simsiam}, VICReg \cite{Bardes2021_vicreg}). This choice has typically been motivated 
by collapse prevention or by the presence of conditional context. Our analysis adds a complementary axis to this design choice in the multimodal case: even without conditional information or asymmetric features, the right 
choice between aligning and predicting depends on the noise structure of 
the modality pair.


\section{Cross-Alignment and Cross Prediction Approaches}\label{sec:solutions_and_bottlenecks}
We study two objectives for multimodal representation learning. The first is \emph{cross-alignment} (CA), which aligns paired samples in a shared latent space. The second is \emph{cross-prediction} (CP), which predicts one modality from the other through an encoder--decoder factorization. Both are formalized below and analyzed throughout the paper. More details and proofs are deferred to Appendix~\ref{app:closed_form}.

\subsection{Objectives}
\label{sec:objectives}
Let $\fx : \R^{d_x} \to \R^k$ and $\fy : \R^{d_y} \to \R^k$ be
encoders producing latent codes $\vz_x^{(i)} := \fx(\vx_i)$ and
$\vz_y^{(i)} := \fy(\vy_i)$ in a shared latent space of dimension $k$,
and let
$\bar\vz_x^{(i)} := \vz_x^{(i)} - \tfrac{1}{n}\sum_j \vz_x^{(j)}$ and
$\bar\vz_y^{(i)} := \vz_y^{(i)} - \tfrac{1}{n}\sum_j \vz_y^{(j)}$
denote the centered codes.
Let $\fD : \R^k \to \R^{d_y}$ be a decoder. The two objectives are
\begin{align}
\text{(CA)} \quad &\min_{\fx, \fy} \;
  \tfrac{1}{n} \textstyle\sum_i \| \fx(\vx_i) - \fy(\vy_i) \|_2^2
  \quad \text{s.t.} \quad
  \tfrac{1}{n} \textstyle\sum_i \bar\vz_x^{(i)} \bar\vz_x^{(i)\top}
  = \tfrac{1}{n} \textstyle\sum_i \bar\vz_y^{(i)} \bar\vz_y^{(i)\top}
  = \mI_k,
  \label{eq:ca}
\\
\text{(CP)} \quad &\min_{\fx, \fD} \;
  \tfrac{1}{n} \textstyle\sum_i \| \vy_i - \fD(\fx(\vx_i)) \|_2^2.
  \label{eq:cp}
\end{align}
The CA constraint fits the cross-modal similarity matrix to the identity:
matched pairs are pulled together while each view is separately whitened,
so that no coordinate collapses and no two coordinates are redundant.
Whitening each view rather than their cross-moment is what makes
\eqref{eq:ca} reduce exactly to CCA (\Cref{thm:closed_form_ca}); it is
also the form implemented softly by the variance and covariance terms of
VICReg \citep{Bardes2021_vicreg}, which we use as the CA objective in all
non-linear experiments.
\Cref{fig:ca_cp_paradigms} illustrates the two paradigms.
\begin{figure}[h]
    \centering
    \begin{minipage}[t]{0.55\linewidth}
        \vspace{0pt}
        \centering
        \resizebox{\linewidth}{!}{%
        \begin{tikzpicture}[scale=0.8, every node/.style={transform shape},
            inputnode/.style={circle, draw, minimum size=9mm, inner sep=0pt, font=\small},
            latentnode/.style={circle, draw, fill=gray!30, minimum size=7mm, inner sep=0pt, font=\small},
            encoder/.style={
                draw, fill=gray!15,
                trapezium, trapezium angle=68,
                trapezium stretches body,
                minimum height=9mm, minimum width=15mm,
                shape border rotate=270,
                inner sep=1pt, font=\small\bfseries
            },
            decoder/.style={
                draw, fill=gray!15,
                trapezium, trapezium angle=68,
                trapezium stretches body,
                minimum height=9mm, minimum width=15mm,
                shape border rotate=90,
                inner sep=1pt, font=\small\bfseries
            },
            arr/.style={-{Stealth[length=5pt]}, thick},
            darr/.style={{Stealth[length=5pt]}-{Stealth[length=5pt]}, thick, dashed},
            dasharr/.style={-{Stealth[length=5pt]}, thick, dashed},
            ]
            \node[font=\large\bfseries] at (2.8, 3.0) {Cross-Prediction (CP)};
            \node[inputnode] (x_cp) at (0, 1.2) {$\vx_i$};
            \node[inputnode] (y_cp) at (0, -0.5) {$\vy_i$};
            \node[encoder] (enc_cp) at (2.0, 1.2) {$\fx$};
            \node[latentnode] (z_cp) at (3.4, 1.2) {$\vz_x^{(i)}$};
            \node[decoder] (dec_cp) at (4.8, 1.2) {$f_{\mD}$};
            \draw[arr] (x_cp) -- (enc_cp);
            \draw[arr] (enc_cp) -- (z_cp);
            \draw[arr] (z_cp) -- (dec_cp);
            \draw[dasharr] (dec_cp.east) -- ++(0.4, 0) |- (y_cp.east);
            \draw[gray!50, thick, dashed] (6.8, -1.8) -- (6.8, 3.5);
            \node[font=\large\bfseries] at (10.2, 3.0) {Cross-Alignment (CA)};
            \node[inputnode] (x_ca) at (7.8, 1.5) {$\vx_i$};
            \node[encoder] (enc_w) at (9.6, 1.5) {$\fx$};
            \node[latentnode] (zx) at (11.4, 1.5) {$\vz_x^{(i)}$};
            \node[inputnode] (y_ca) at (7.8, -0.8) {$\vy_i$};
            \node[encoder] (enc_v) at (9.6, -0.8) {$\fy$};
            \node[latentnode] (zy) at (11.4, -0.8) {$\vz_y^{(i)}$};
            \draw[arr] (x_ca) -- (enc_w);
            \draw[arr] (enc_w) -- (zx);
            \draw[arr] (y_ca) -- (enc_v);
            \draw[arr] (enc_v) -- (zy);
            \draw[darr] (zx) -- (zy);
            \node[draw, rounded corners=4pt, dashed, gray!70, thick,
                  fit=(zx)(zy),
                  inner xsep=10pt, inner ysep=10pt] (sharedbox) {};
            \node[font=\footnotesize, gray!70, anchor=south] at (sharedbox.north) {shared latent space $\R^k$};
        \end{tikzpicture}%
        }
    \end{minipage}%
    \hfill
    \begin{minipage}[t]{0.44\linewidth}
        \vspace{0pt}
        \caption{Two multimodal learning paradigms studied in this work. \figleft~\textbf{Cross-prediction} (CP), ~\Cref{eq:cp}: an encoder $\fx$ maps modality $\vx$ to a latent code $\vz$, and a decoder $f_{\mD}$ reconstructs the paired target $\vy$. \figright~\textbf{Cross-alignment} (CA), ~\Cref{eq:ca}: encoders $\fx$ and $\fy$ project paired samples $(\vx_i, \vy_i)$ into a shared latent space where matched pairs are pulled together.}
        \label{fig:ca_cp_paradigms}
    \end{minipage}
\end{figure}
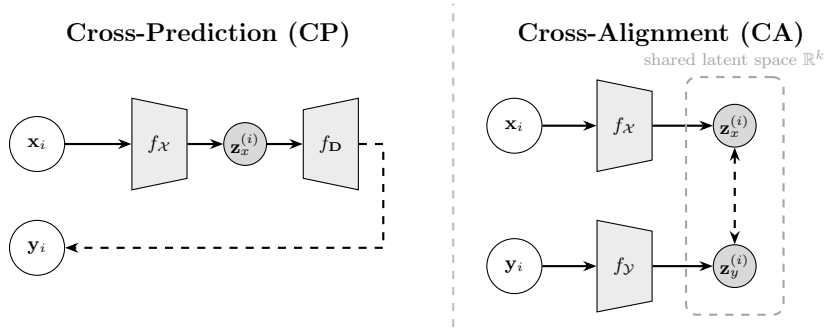
To understand when each objective succeeds or fails, we restrict to the
linear case, where \eqref{eq:ca} and \eqref{eq:cp} admit closed-form
solutions and the population geometry is fully tractable. This linear
analysis captures the core mechanism and provides our main analysis tool, and its
predictions transfer to the nonlinear regime in our experiments (\Cref{sec:experiments}).

\subsection{Linear analysis under a spiked model}
\label{sec:linear}

With linear encoders $f_X(\vx) = \mW\vx$, $f_Y(\vy) = \mV\vy$ for CA, and linear 
encoder $f_X(\vx) = \mE\vx$ and decoder $f_D(\vz) = \mD\vz$ for CP, both 
objectives admit closed-form solutions expressible through the SVDs of two 
modality-coupling matrices:
\begin{equation}
\mC := \mSxx^{-1/2} \mSxy \mSyy^{-1/2}, \qquad
\mA := \mSyx \mSxx^{-1/2},
\label{eq:coupling_matrices}
\end{equation}
where $\mSxx$, $\mSyy$, $\mSxy$ denote the population (cross-)covariances; throughout,
$\vx$ and $\vy$ are zero-mean, so second moments coincide with covariances, and
$\hat{\mSigma}$ denotes sample estimates (\Cref{alg:phase_estimation}). CA
projects onto the leading $k$ singular directions of $\mC$ (symmetric
whitening, equivalent to CCA \citep{Hotelling1936_CCA, Andrew2013_deepCCA});
CP projects onto the leading $k$ singular directions of $\mA$
(source-side whitening only, equivalent to truncated reduced-rank
regression (RRR) \citep{izenman1975, eckart_approximation_1936}). Although CCA and RRR
are classically connected (e.g., \cite{Donnat2024_CCA_RRR}), the two paradigms diverge once
structured cross-modal nuisance is present, which is the regime our
analysis characterizes.

\paragraph{Spiked model.}
To analyze recovery, we posit a signal-plus-noise model in which each
modality decomposes into $k$ shared signal coordinates and $d - k$
modality-specific nuisance coordinates. In suitable orthogonal bases, the
covariances are block diagonal:
\begin{equation}
\mSxx \;=\; \mathrm{diag}\!\bigl(\mK^2 + \mGamma_x^{(s)},\; \mGamma_x^{(n)}\bigr),
\quad
\mSyy \;=\; \mathrm{diag}\!\bigl(\mK^2 + \mGamma_y^{(s)},\; \mGamma_y^{(n)}\bigr),
\quad
\mSxy \;=\; \mathrm{diag}\!\bigl(\mK^2,\; \mGamma_{xy}\bigr),
\label{eq:spiked_model}
\end{equation}
where $\mK = \diag(\kappa_1, \ldots, \kappa_k)$ collects the cross-modal
signal strengths, $\mGamma_x^{(s)}, \mGamma_y^{(s)}$ are the view-specific
noise variances on the shared coordinates, $\mGamma_x^{(n)}, \mGamma_y^{(n)}$ are the view-specific noise variances on the nuisance  coordinates, and $\mGamma_{xy} = \diag(\eta_1, \ldots, \eta_{d-k})$ encodes
cross-modal nuisance correlation, with
$0 \leq \eta_j \leq \sqrt{\tilde\gamma_j^x \tilde\gamma_j^y}$. Full parameterization
is given in \Cref{app:closed_form}.

Under this model, the singular values of $\mC$ and $\mA$ decompose cleanly
into signal and nuisance contributions, yielding recovery conditions for
each objective.

\begin{proposition}[CA vs.\ CP separation]
\label{prop:separation}
Under the spiked model \eqref{eq:spiked_model}, the singular values of
$\mC$ and $\mA$ split into signal values
$\{\rho_i, \tau_i\}_{i \in \llbracket k \rrbracket}$ and nuisance values
$\{\nu_j, \xi_j\}_{j \in \llbracket d-k \rrbracket}$, given by
\begin{equation}
\rho_i = \frac{\kappa_i^2}{\sqrt{(\kappa_i^2 + \gamma_i^x)(\kappa_i^2 + \gamma_i^y)}},
\quad
\tau_i = \frac{\kappa_i^2}{\sqrt{\kappa_i^2 + \gamma_i^x}},
\quad
\nu_j = \frac{\eta_j}{\sqrt{\tilde\gamma_j^x \tilde\gamma_j^y}},
\quad
\xi_j = \frac{\eta_j}{\sqrt{\tilde\gamma_j^x}}.
\label{eq:singular_values}
\end{equation}
CA (resp.\ CP) recovers the shared signal subspace whenever its signal
singular values exceed its nuisance singular values. Defining the
\emph{separation ratios}
\begin{equation}
\DCA := \frac{\min_i \rho_i}{\max_j \nu_j}, \qquad
\DCP := \frac{\min_i \tau_i}{\max_j \xi_j},
\label{eq:separation_ratios}
\end{equation}
full recovery holds when $\DCA > 1$ (for CA) or $\DCP > 1$ (for CP). In the
homogeneous case ($\kappa_i \equiv \kappa$, $\gamma^y_i \equiv \gamma^y$,
$\tilde\gamma^y_j \equiv \tilde\gamma^y$), the two ratios satisfy
\begin{equation}
\DCA \;=\; \DCP \cdot \sqrt{\frac{\tilde\gamma^y}{\kappa^2 + \gamma^y}}.
\label{eq:ratio_identity}
\end{equation}
More generally, $\DCA/\DCP$ is monotonically non-decreasing in each
$\tilde\gamma^y_j$.
\end{proposition}

\begin{proposition}[Partial recovery]
\label{prop:partial_recovery}
Under the spiked model \eqref{eq:spiked_model}, the top-$k$ singular
vectors of $\mC$ (resp.\ $\mA$) contain at least
\begin{equation}
r_{\mathrm{CA}} := \bigl|\{i : \rho_i > \textstyle\max_j \nu_j\}\bigr|,
\qquad
r_{\mathrm{CP}} := \bigl|\{i : \tau_i > \textstyle\max_j \xi_j\}\bigr|
\end{equation}
shared signal directions.
\end{proposition}

Full recovery is $r = k$ (\Cref{prop:separation}); $r = 0$ corresponds to
complete failure; and the intermediate regime $0 < r < k$ is \emph{partial
recovery}, in which the learned representation is guaranteed to capture
at least $r$ signal directions mixed with nuisance. Partial recovery
arises only under heterogeneous signal strengths: in the homogeneous case
$\kappa_i \equiv \kappa$, $\gamma_i^x \equiv \gamma^x$, $\gamma_i^y \equiv \gamma^y$
(\Cref{fig:phase_diagram}) all signal
singular values coincide, so $r \in \{0, k\}$ and partial recovery
collapses to the binary regime of \Cref{prop:separation}.  As $\kappa_i$ become increasingly heterogeneous, the four-region phase diagram of \figref{fig:phase_diagram} smears into a graded continuum (\Cref{fig:partial_recovery}, ~\Cref{sec:additional_figs}). In particular, the Neither regime ($\Delta_{\mathrm{CA}}, \Delta_{\mathrm{CP}} \le 1$) means
that full recovery fails for both objectives, not that the signal is lost: it necessarily coincides with
complete failure ($r = 0$) only in the homogeneous case.

\paragraph{Failure modes.}
\Cref{eq:singular_values} exposes a fundamental asymmetry. CA's nuisance
values $\nu_j$ are \emph{cross-modal correlation coefficients} in $[0, 1]$,
independent of nuisance variance: when any $\nu_j \to 1$, no signal
direction can match it under whitening (since $\rho_i < 1$ whenever any
modality-specific noise is present), so $\DCA < 1$ regardless of signal
strength. CP's nuisance values $\xi_j = \nu_j \sqrt{\tilde\gamma_j^y}$
depend on the target nuisance variance: large $\tilde\gamma_j^y$ amplifies
even moderate nuisance correlation into a recovery-breaking singular value
(here and in Figure~\ref{fig:phase_diagram}, we hold $\nu_j$ fixed when varying
$\tilde\gamma^y_j$, consistent with the $(\kappa, \nu)$ axes of the phase diagram).
CP is also asymmetric: swapping source and target replaces
$\tilde\gamma_j^y$ with $\tilde\gamma_j^x$, so the direction of prediction
matters. A direct corollary, demonstrated in \Cref{fig:ca_probe_vs_cp}, is that CP can achieve lower reconstruction MSE than CA while recovering the wrong subspace: the MSE objective itself does not distinguish signal from nuisance. The relation of \eqref{eq:ratio_identity} makes the complementarity explicit: the ratio
$\DCA/\DCP$ grows with $\tilde\gamma^y_j$, so a large target-side nuisance favors CA, while a small target-side nuisance favors CP. The same relation predicts when latent-space prediction helps: regressing onto a learned
target representation rather than the raw target, as in data2vec or JEPA-style predictors,
suppresses target-side nuisance variance, shrinking $\tilde\gamma^y_j$ and enlarging the
region where CP recovers. Masked prediction, which changes which target coordinates are scored rather than their variance, is outside the model.
In the resulting phase diagram (\Cref{fig:phase_diagram}) the
$(\kappa, \nu)$ plane partitions into four regions --- \emph{Both},
\emph{CA only}, \emph{CP only}, \emph{Neither} --- with boundaries set by
$\DCA = 1$ and $\DCP = 1$.

\begin{keytakeaway}
    \textbf{Key takeaway.} These failure modes suggest that CA is preferable when nuisance correlation is moderate, modality quality is uncertain, and modality-specific noise is large. CP may be preferable, with the correct source-target orientation, when the signal is strong and the target noise is weak.
\end{keytakeaway}

\begin{figure}[t]
\centering
    \begin{minipage}{0.65\linewidth}
    \includegraphics[width=\linewidth]{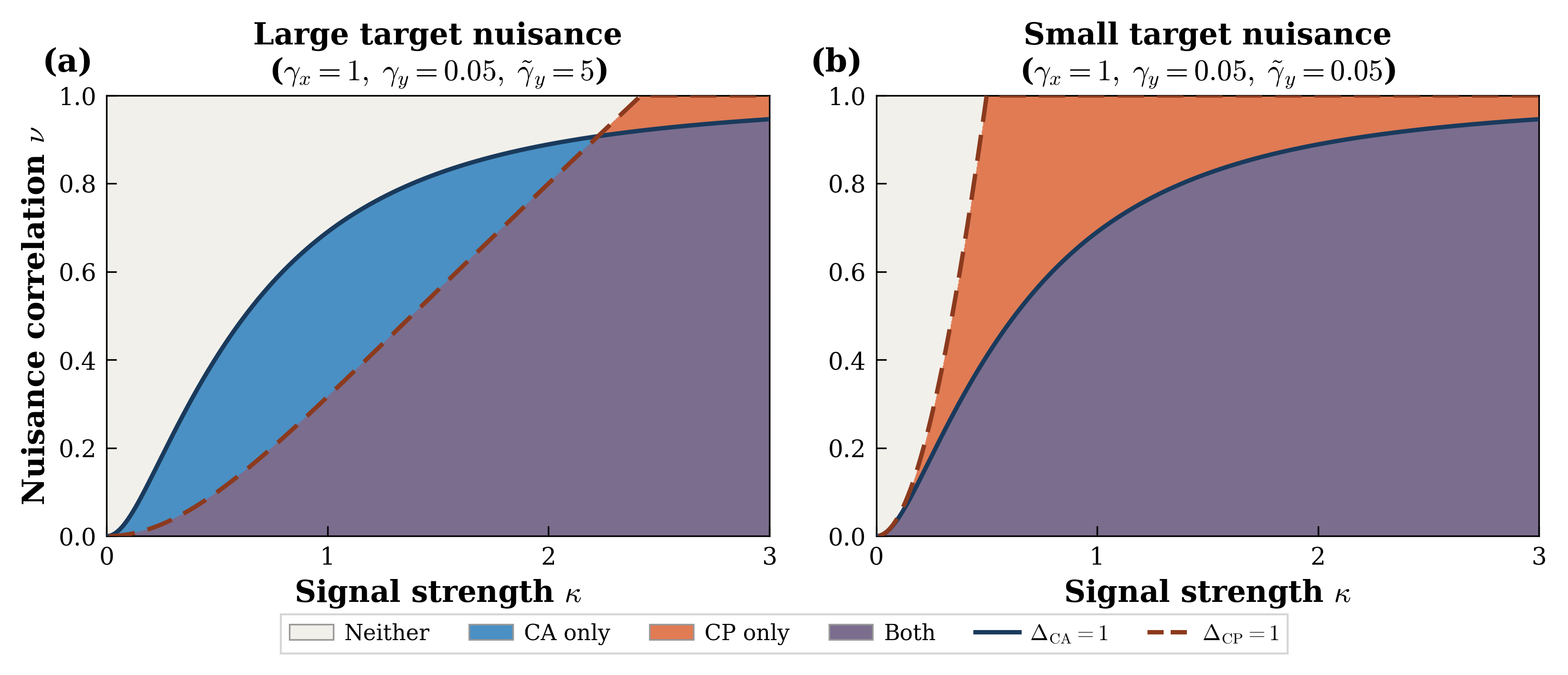}  
    \end{minipage}
    \begin{minipage}{0.34\linewidth}
    
    \caption{Phase diagram for signal recovery in $(\kappa, \nu)$ space under the homogeneous model (all signal and noise components are equal). Solid and dashed lines respectively show the $\DCA = 1$ and $\DCP = 1$ boundaries from \Cref{prop:separation}.
    (a) Large target nuisance ($\tilde\gamma^y \gg \gamma^y$). (b) Small target noise ($\tilde\gamma^y \sim \gamma^y$). Phase diagrams for the non-homogeneous case with partial recoveries are shown in ~\Cref{fig:partial_recovery}.}
    \label{fig:phase_diagram}
    \end{minipage}
\end{figure}

\paragraph{Redundant vs.\ Complementary Modalities.}
The failure modes of CA and CP can be understood through a distinction between two regimes of multimodal data. \emph{Redundant} modalities, such as image--caption 
pairs, where captions are written to describe image content, share dominant structure across views, corresponding to large $\kappa_i$ and small $\nu_j$ in 
our model, or the lower right corner in ~\Cref{fig:phase_diagram}. In this regime, the redundancy assumption underlying standard multimodal SSL holds. \emph{Complementary} modalities, by contrast, arise when each view provides a structurally distinct perspective on the same object, as in multi-sensor measurements in astrophysics and earth science, or multi-omics or multi-scale profiling in biology. Here, view-specific structure dominates the variance of each modality while cross-modal nuisance correlation is non-negligible, pushing toward large $\nu_j$ and small effective $\kappa_i$, precisely the Neither region (upper left corner) of Figure~\ref{fig:phase_diagram} where both paradigms fail. The intermediate region, which is represented by the center area in Figure~\ref{fig:phase_diagram} is noise dependent - CA is preferred when target noise is large, and CP is preferred when target noise is weak. 
\section{Experiments}\label{sec:experiments}
To test the validity of our theory in real-world problems, we conducted experiments across varying levels of complexity and controllability. The first is a linear experiment that directly implements the spiked model, followed by non-linear experiments with synthetic vision datasets, simulating multi-modality by projecting the scene onto two virtual cameras. Next, we use MS-COCO, a real image-caption multi-modality dataset. Finally, we verify our
predictions on two real scientific problems: an astrophysics multimodal experiment, and a
single-cell multi-omics pair. In all non-linear experiments, we use the VICReg approach \citep{Bardes2021_vicreg} as the CA objective, and MSE reconstruction as CP.
VICReg's variance and covariance terms are the soft form of the constraint
in \eqref{eq:ca}: the variance term keeps each coordinate of a view from
collapsing, and the covariance term decorrelates coordinates within a view,
together penalizing departures from $\tfrac{1}{n}\sum_i \bar\vz^{(i)} \bar\vz^{(i)\top} = \mI_k$, while the invariance term is the alignment loss itself. VICReg was shown to be a
formulation of DeepCCA \citep{Andrew2013_deepCCA}, the non-linear version of CCA
\citep[see e.g.,][]{Chapman2023_unified_cca}, and we compare the two
directly in our two synthetic experiments (see \Cref{sec:additional_figs}). We provide implementation details of all experiments in ~\Cref{sec:implementation_details}.


\subsection{Spiked Synthetic Data}\label{sec:linear_experiments} 

We first verify the theoretical predictions of \Cref{sec:solutions_and_bottlenecks} using closed-form solvers on finite-sample, synthetic data drawn from the spiked covariance model. This confirmed that cross-modal noise correlation breaks CP well before CA (Figure~\ref{fig:linear_E4}). The  empirical subspace distance shows that CP fails to recover the signal ($\text{dist} \to 1$) at $\nu \approx 0.15$, while CA maintains near-perfect recovery until $\nu \approx 0.75$ (Figure~\ref{fig:linear_E4}, left). Based on the theoretical separation ratios (Figure~\ref{fig:linear_E4}, right), $\Delta_{\mathrm{CP}}$ crosses the recovery threshold $\Delta = 1$ at a noise correlation roughly $5\times$ lower than $\Delta_{\mathrm{CA}}$. The wide gap between the two thresholds validates the complementary failure modes identified in ~\Cref{sec:solutions_and_bottlenecks}. Replacing the CA solver with an InfoNCE-trained linear encoder, as in CLIP-style training,
does not change this picture: InfoNCE fails only at similarly high nuisance correlation, far from CP (\Cref{app:infonce}), since negatives break pairing across samples
while nuisance correlation is a within-pair property. We provide additional linear experiments in \Cref{sec:additional_figs}.

\subsection{Stereo vision experiments}
\textbf{Stereo-dSprites}
We next tested whether the complementary failure modes persist in a nonlinear setting using \emph{Stereo-dSprites}, a synthetic stereo-vision benchmark, based on the dSprites dataset (\cite{dsprites17}), with controlled nuisance alignment that simulates a multimodal visual setting. Two virtual cameras observe a shared 2D object on a $64 \times 64$ image. Here, shape is the signal (low pixel variance but perfectly correlated across views) whereas world position is the nuisance (high pixel variance and highly, but imperfectly, correlated, with the correlation controlled by a camera jitter parameter $\sigma_{\text{jitter}}$). We define nuisance alignment as $\nu_{\max} = 1 - \sigma_{\text{jitter}}$, and evaluate downstream shape classification via linear probe (details in \Cref{sec:implementation_details}).\\
\textbf{Stereo-3DShapes}
We replicate the Stereo-dSprites protocol on \emph{Stereo-3DShapes} (based on \cite{3dshapes18}): RGB $64 \times 64$ stereo pairs of 3D-rendered objects (Cube, Cylinder, Sphere, Capsule) with controlled position jitter (details in \Cref{sec:implementation_details}). 

\begin{figure}[t]
    \centering
    \begin{minipage}{0.55\linewidth}
    \includegraphics[width=\linewidth]{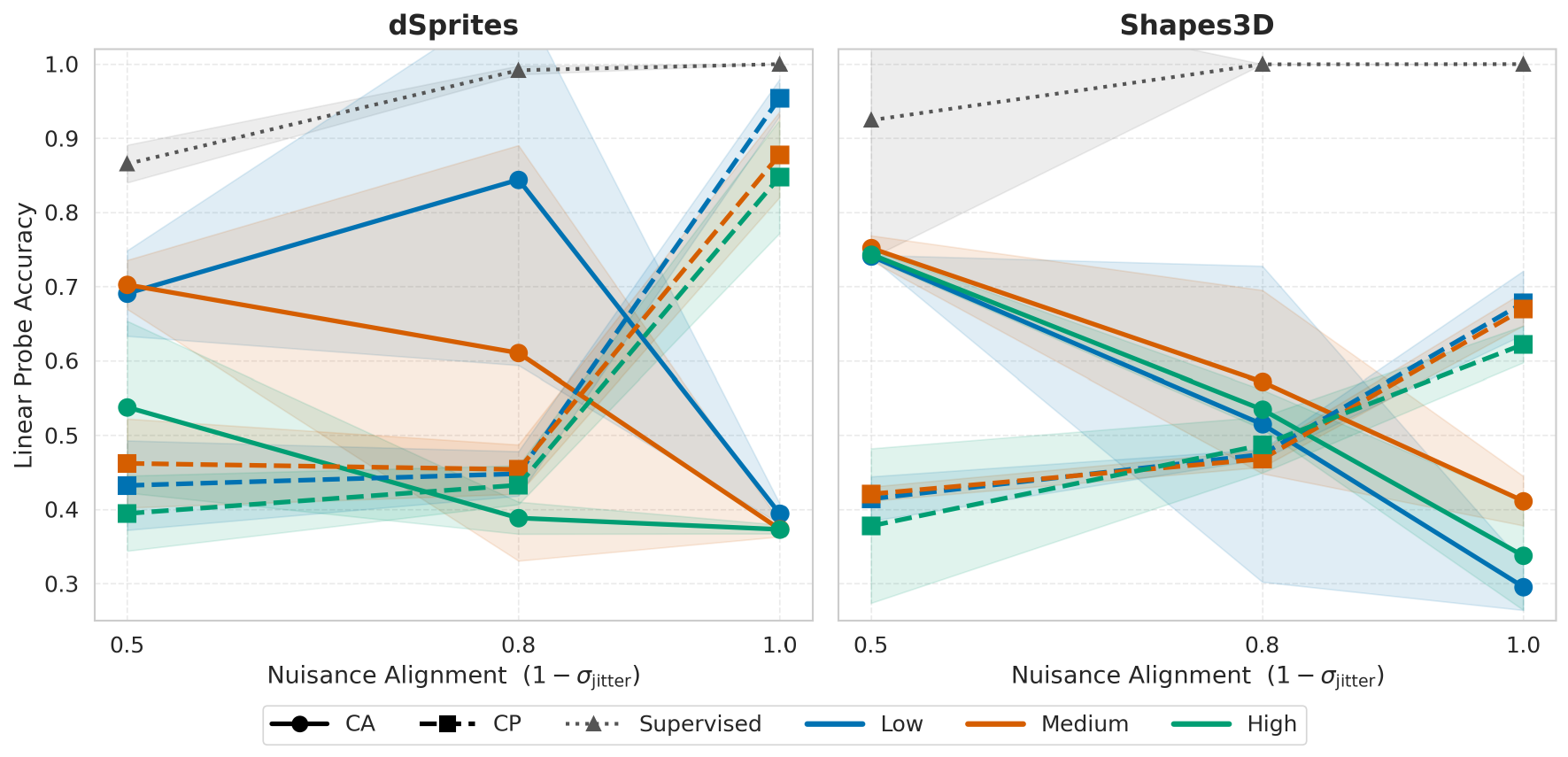}  
    \end{minipage}
    \begin{minipage}{0.44\linewidth}
    \caption{Linear probe accuracy vs.\ nuisance alignment $\nu_{\max} = 1 - \sigma_{\text{jitter}}$. \figleft~Stereo-dSprites (3-class, grayscale, 100k samples). \figright~Stereo-3DShapes (4-class, RGB, 100k samples). In both settings, color represents weak noise levels. In both panels, the trade-off between the methods is clearly seen. CA (solid, circles) peaks at moderate-to-low alignment and collapses at full alignment; CP (dashed, squares) shows the opposite pattern. The crossover at $\nu_{\max} \approx 0.8$ is consistent across datasets. Lower absolute ceilings in 3DShapes reflect the harder discrimination task.}
    \label{fig:stereo_performance}
    \end{minipage}
\end{figure}

Both experiments shows the expected trade-off between objectives - CA fails with perfectly correlated noise and improves as the alignment decreases, while CP shows the opposite behavior (Figure \ref{fig:stereo_performance}). At full nuisance alignment ($\nu=1$), the cross-modal mapping is deterministic and CP's overcapacity bottleneck encodes both signal and nuisance without compression pressure, circumventing the theory's failure prediction. As soon as jitter breaks determinism ($\nu < 1$), CP is forced to compress, and the separation ratio governs recovery. These overall conclusions are also observed when we examine examples of latent features with different alignment values for CA and CP in the stereo-dSprites experiment with dimensionality reduction using UMAP~\citep{McInnes2018_umap} ((Figure \ref{fig:dsprites_umap}), colors: signal (shape); opacity: nuisance features (position)).     

\begin{figure}[t]
    \centering
    \centering
    \includegraphics[width=\linewidth]{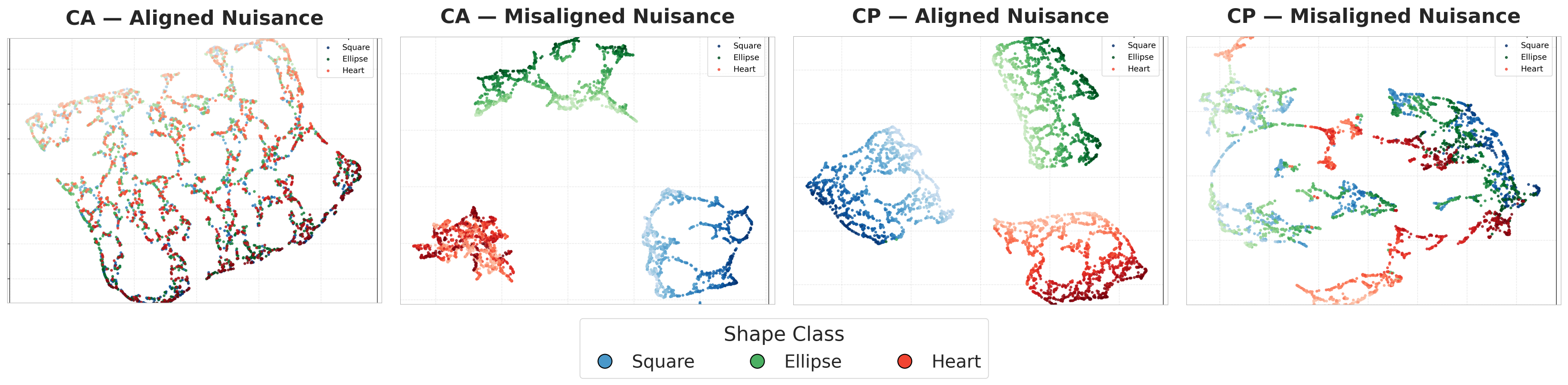}
    \caption{UMAP embeddings of learned representations of the stereo-dSprites experiment (color $=$ shape, intensity $=$ position). From left to right: CA with aligned noise ($\sigma_{\text{jitter}}=0$), CA with misaligned noise ($\sigma_{\text{jitter}}=0.5$), CP with aligned noise, and CP with misaligned noise. All experiments have the same modality-specific noise ($\sigma_{\text{noise}}=0.5$). Each method succeeds exactly where the other fails, and on failures, the models learn the nuisance.}
    \label{fig:dsprites_umap}
\end{figure}

\subsection{Image--Caption Experiments}
\label{sec:image_caption}

We extended the analysis to a real image--caption dataset: MS-COCO~\citep{lin2014_mscoco}, where captions are written to describe natural image content. As this is a true multimodal dataset, we used the natural caption--image pairing without artificial nuisance manipulation. 
We trained encoders from scratch (ResNet-18 for images, two-layer Transformer 
for captions) and varied modality-specific noise by applying visual style transforms 
of increasing strength to the image modality while keeping captions clean.
As expected, we observed asymmetric performance for CP (Figure~\ref{fig:coco_capswap}); 
$\mathrm{CP}_{I \to T}$ strongly dominates at low noise (${\sim}10$\,pp over CA), and 
degrades monotonically as image noise increases. The asymmetry matches the linear theory: the target covariance $\mSyy$ never enters $\mathbf{A}$, so image-side perturbations affect $\mathrm{CP}_{I \to T}$ but leave $\mathrm{CP}_{T \to I}$ (text source, unchanged) flat.
The slight rise of $\mathrm{CP}_{T \to I}$ is consistent with noisy reconstruction 
targets acting as a regularizer, a finite-capacity effect outside the linear theory. 
CA is insensitive throughout, consistent with $\mSigma_{yy}$ normalization absorbing 
modality-specific variance.

\subsection{Predicting Recovery Regimes}
\label{sec:regime_prediction}

We propose a lightweight supervised diagnostic that locates a paired dataset in the phase diagram before committing to cross-modal training. The diagnostic is supervised by design: a small labeled subsample 
classifies singular components of $\hat\C$ and $\hat\Ab$ as signal or nuisance, allowing direct estimation 
of $\hat\DCA$ and $\hat\DCP$ without recovering the latent parameters $(\kappa, \gamma, \tilde\gamma, \eta)$. The labeled budget is small relative to the scale of cross-modal training. The diagnostic is meant 
to inform, and the labels need not match the downstream target: in \Cref{sec:astro}, regime prediction from $\log g$ correctly orders 
methods on binarity and age. Under subsampling down to a few hundred labeled pairs, most draws reproduce the full-sample
regime and the recovery count is stable throughout (\Cref{app:subsample}). The predicted
regime is also unchanged under label noise, partial reassignment, and class coarsening
(\Cref{app:label_robustness}).

The estimation procedure is well-posed when the representation fed to the estimator approximately satisfies the linear spiked decomposition
of \Cref{sec:linear}: signal and nuisance directions must be meaningfully separable in the joint covariance structure of the paired
data. This condition holds most directly in \emph{two-stage} multimodal pipelines, where each modality is first encoded independently by a unimodal model, and the cross-modal objective operates on these frozen representations. Two-stage pipelines are the dominant paradigm in scientific
multimodal learning and in large foundation-model stacks where modality-specific encoders are pretrained independently. In \emph{single-stage} pipelines, the encoder and cross-modal objective are optimized jointly from raw data, and no unimodal representation
satisfying the spiked decomposition exists prior to training. Applying the estimator to compact proxies of the raw inputs works when those inputs are
already low-dimensional and informative, as for the protein modality in \Cref{sec:bmmc}. 

\begin{figure}[t]
    \centering
        \centering
        \begin{minipage}{0.29\linewidth}
        \includegraphics[width=\linewidth]{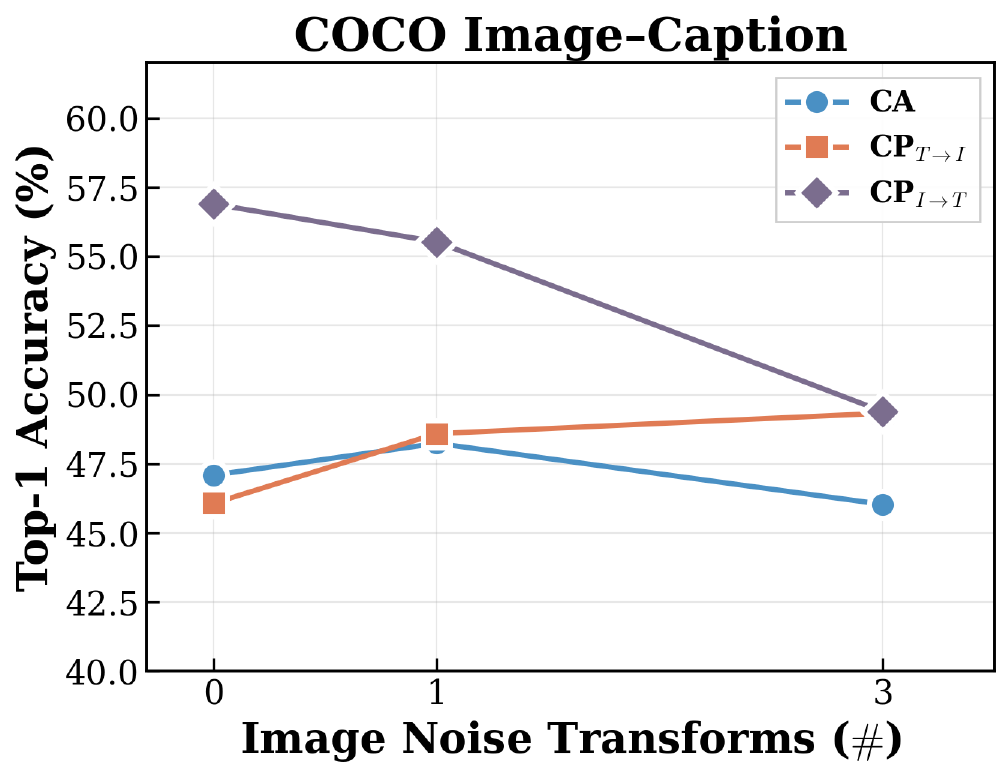}
        \end{minipage}
    \begin{minipage}{0.70\linewidth}
    \caption{Top-1 accuracy vs.\ image style transform strength for MS-COCO experiment. CP shows an asymmetric nature: prediction of image from text results in similar performance as CA but prediction of text from an image results in much better performance. Both approaches converge to the same accuracy when image noise is high. Single run per condition.}
    \label{fig:coco_capswap}
    \end{minipage}
\end{figure}

\subsection{Real Astrophysical Data (LAMOST $\times$ Kepler/TESS)}
\label{sec:astro}
Finally, we validated regime estimation on real astrophysical data, pairing
ground-based LAMOST (\cite{Zhao2012_LAMOST}) spectra (2048-dim encoder) with space-based photometry
from two instruments: Kepler (\cite{Mathur_2017_kepler_dr25}) and TESS (\cite{Ricker2015_TESS}) (1024-dim encoders), using frozen pretrained features \cite{Kamai2025_desa} with lightweight projection/prediction
heads (\Cref{sec:implementation_details}). We estimated recovery regimes
from a labeled subsample using surface gravity ($\log g$), following
\Cref{alg:phase_estimation}. Downstream evaluation covered three physically 
distinct targets: binarity, $\log g$, and age. Binarity and age are encoded through modality-specific 
mechanisms (spectroscopic radial-velocity variations \textit{vs}.\ photometric eclipses for binarity; isochrone \textit{vs}. gyrochronology for age), so agreement between a $\log g$-based regime prediction and behavior on binarity and age 
tests whether the separation ratios capture the geometry of the modality pair rather 
than a task-specific artifact.
The results show that the regime predictions hold as \emph{regime-level} statements across all targets, with a
revealing asymmetry in how each regime manifests (\Cref{tab:astro_results}, 5 seeds per cell). 

\emph{Kepler (Both).} On every target, at least one of \{CA, CP\} matches or
exceeds the best unimodal baseline, but the winner rotates with task.
CP stays at the LAMOST ceiling where LAMOST dominates ($\log g$, binarity); CA captures photometric signal where
LAMOST is weak (age: $+0.19$ $R^2$ over LAMOST). The 'Both' prediction should be read as \emph{some cross-modal 
objective helps on every task}, not as both objectives helping uniformly,
consistent with $\hat{\DCA}$ and $\hat{\DCP}$ being very different.

\emph{TESS (Neither).} No cross-modal method exceeds LAMOST alone on any
target, and the gap is well beyond seed variance. Moreover, we have
$\hat r_{\mathrm{CA}} = \hat r_{\mathrm{CP}} = 0$, the complete-failure form of the
Neither regime.
\Cref{fig:phase_diagnostics} visualizes the underlying singular-value 
decompositions of $\hat\C$ and $\hat\Ab$ for both pairs.

\paragraph{CP direction asymmetry.}CP's recovery condition $\Delta_{\mathrm{CP}}$ never involves the target covariance $\mSyy$, which does not enter $\mathbf{A}$, so swapping source and target
produces a structurally different recovery condition: the preferred direction is the one 
in which the source modality more directly encodes the task signal. For $\log g$ 
and binarity, spectra encode the signal more directly than photometry (through 
absorption-line broadening rather than granulation/oscillation) so CP forward 
(spectra $\to$ photometry) succeeds while CP$_{\rm rev}$ (photometry $\to$ 
spectra) fails. For age, photometric rotation periods provide a more direct 
signal via gyrochronology than spectroscopic activity indicators, so the 
preferred direction reverses: CP$_{\rm rev}$ on age ($R^2 = 0.497$) outperforms 
forward CP.

\begin{table}[h]
\centering
\caption{Astrophysical cross-modal results, mean $\pm$ std over 5 seeds.
Best per row in \textbf{bold} (ties within seed std co-bolded).
\emph{Kepler (Both, $\hat\Delta_{\mathrm{CA}}=1.08$, $\hat\Delta_{\mathrm{CP}}=2.29$):} at least
one cross-modal method matches or beats the best unimodal baseline on
every target; CP preserves LAMOST's ceiling where LAMOST dominates,
CA captures photometric signal where LAMOST is weak.
\emph{TESS (Neither, $\hat r_{\mathrm{CA}} = \hat r_{\mathrm{CP}} = 0$):} no cross-modal
method beats LAMOST-only on any target. CP$_{\rm rev}$ (photometry $\to$ spectra) fails on tasks where spectra 
carry the more direct signal ($\log g$, binarity), but outperforms forward 
CP on age, where photometric rotation provides a more direct gyrochronological 
signal — the same source-quality principle in both directions.}
\label{tab:astro_results}
\small
\setlength{\tabcolsep}{4pt}
\begin{tabular}{l c c c c c}
\toprule
Task & CA & CP & CP$_\mathrm{rev}$ & LAMOST & Photometry \\
\midrule
\multicolumn{6}{l}{\textbf{LAMOST $\times$ Kepler} --- Both regime} \\
\midrule
Binarity (bal.\ acc.) & $0.802{\pm}0.009$ & $\mathbf{0.814{\pm}0.006}$ & $0.751{\pm}0.004$ & $\mathbf{0.814}$ & $0.731$ \\
$\log g$      ($R^2$) & $0.956{\pm}0.003$ & $\mathbf{0.976{\pm}0.001}$ & $0.639{\pm}0.004$ & $\mathbf{0.977}$ & $0.542$ \\
Age           ($R^2$) & $\mathbf{0.620{\pm}0.001}$ & $0.434{\pm}0.006$ & $0.497{\pm}0.039$ & $0.431$ & $0.470$ \\
\midrule
\multicolumn{6}{l}{\textbf{LAMOST $\times$ TESS} --- Neither regime} \\
\midrule
Binarity (bal.\ acc.) & $0.756{\pm}0.022$ & $0.763{\pm}0.011$ & $0.626{\pm}0.010$ & $\mathbf{0.779}$ & $0.604$ \\
$\log g$      ($R^2$) & $0.929{\pm}0.005$ & $0.939{\pm}0.001$ & $-0.312{\pm}0.001$ & $\mathbf{0.942}$ & $-0.312$ \\
Age           ($R^2$) & $0.431{\pm}0.029$ & $0.396{\pm}0.064$ & $-0.072{\pm}0.004$ & $\mathbf{0.503}$ & $-0.037$ \\
\bottomrule
\end{tabular}
\end{table}

\subsection{Single-Cell Multi-Omics (BMMC CITE-seq)}
\label{sec:bmmc}

We apply the same protocol to a single-cell multi-omics pair with a structurally different
source of correlated nuisance. BMMC CITE-seq measures RNA expression and surface-protein
abundance in the same $90$k bone-marrow mononuclear cells \citep{Luecken2021ASF}, with $44$
annotated cell types as the downstream target; donor, site, and batch effects contaminate
both assays of a cell and are therefore shared across the pair. RNA is encoded by a frozen
scGPT model \citep{cui_scgpt_2024} and the protein modality enters as CLR-normalized counts
(\Cref{par:bmmc_implementation}).

The diagnostic predicts \emph{Neither} in all five seeds
($\hat\Delta_{\mathrm{CA}} = 0.52 \pm 0.13$, $\hat\Delta_{\mathrm{CP}} = 0.33 \pm 0.10$), but
unlike TESS the recovery counts are non-zero ($0 < \hat r < \hat k$ in both spectra in every
seed): partial recovery in the sense of \Cref{prop:partial_recovery}. The downstream results
track this reading (\Cref{tab:bmmc}): the predicted ordering
$\mathrm{CA} > \mathrm{CP} > \mathrm{CP}_{\mathrm{rev}}$ holds in every seed, CP falls below
the RNA-only baseline in $4/5$ seeds, consistent with the variance trap, and CA exceeds it by
less than its own seed-to-seed spread. The two Neither cases thus separate the regime's two
manifestations: $\hat r = 0$ is complete failure, where the stronger modality alone is the
best available representation, while $0 < \hat r < \hat k$ leaves a partially recovered
subspace whose downstream value the separation ratios do not by themselves determine. The
nuisance floor is set by donor, site, and batch structure, and removing the dominant batch
direction promotes the next rather than lifting $\hat\Delta_{\mathrm{CA}}$ above one
(\Cref{app:bmmc_floor}).

\begin{keytakeaway}
   \textbf{Key takeaway.} Estimating effective recovery regimes is a practical and feasible analysis for real-world multimodal problems. It can indicate whether cross-modal learning is likely to succeed, identify which modality is the informative bottleneck, and guide the
    choice of objective.
\end{keytakeaway}

\begin{table}[H]
\centering
\caption{BMMC CITE-seq, $44$-class cell-type balanced accuracy, mean $\pm$ std over $5$
seeds. Last column: within-seed difference from RNA-only, with the number of seeds in which
it is positive. Predicted regime: Neither with partial recovery.}
\label{tab:bmmc}
\small
\begin{tabular}{lcc}
\toprule
Representation & Balanced accuracy & $\Delta$ vs.\ RNA-only \\
\midrule
RNA only (strong) & $0.584 \pm 0.043$ & --- \\
ADT only (weak)   & $0.558 \pm 0.040$ & --- \\
\midrule
CA                & $0.606 \pm 0.030$ & $+0.023 \pm 0.026$ \quad ($4/5$) \\
CP                & $0.572 \pm 0.034$ & $-0.011 \pm 0.011$ \quad ($1/5$) \\
$\mathrm{CP}_{\mathrm{rev}}$ & $0.524 \pm 0.048$ & $-0.060 \pm 0.024$ \quad ($0/5$) \\
\bottomrule
\end{tabular}
\end{table}

\section{Conclusion}\label{sec:conclusion}
We studied cross-modal alignment and cross-modal prediction in a unified linear framework,
deriving recovery conditions governed by separation ratios $\DCA$ and $\DCP$ that
partition multimodal problems into four regimes, refined by the recovery count $r$. A data-driven estimation of these quantities identifies the
preferred objective and prediction direction from a small labeled subsample, before any
cross-modal training, and identifies the case $r = 0$ in which no objective in the CA/CP
family recovers shared signal above the nuisance floor.
Experiments span synthetic data, stereo-vision benchmarks, and image--caption pairs, with
the sharpest validation on real scientific data: the same spectroscopic encoder paired
with two photometric instruments of differing quality yields two distinct predicted
regimes, both confirmed across multiple downstream targets, and the predicted CP direction
asymmetry is confirmed on all tasks. A single-cell multi-omics pair, where the shared nuisance is donor and batch structure rather than instrument systematics, reproduces the Neither prediction in a structurally different domain.\\
The Neither regime is the most important open problem raised by this work, the natural
habitat of complementary scientific modalities, where each view is structurally distinct
yet the shared signal does not clear the nuisance floor. Our two instances show its two
faces: complete failure ($r = 0$), where the stronger modality alone is the best available
representation, and partial recovery ($0 < r < k$), where some signal is recovered but
full recovery fails and the downstream value of the recovered subspace is not settled by
the separation ratios. Escaping it likely requires objectives that go beyond pairwise cross-covariance: higher-order structure, auxiliary supervision, or modality-specific priors, rather than more expressive fusion heads, which can at best learn to rely on the stronger modality. We hope the phase diagram introduced here provides a principled starting point to
solve it.\\
\textbf{Limitations.}
The recovery conditions are derived for linear encoders and decoders at the population
level; the nonlinear experiments test whether the mechanisms persist but do not extend the
guarantees, and finite-sample thresholds remain open. The spiked model assumes orthogonal
signal and nuisance blocks, which real data satisfy only approximately, and the stereo
benchmarks control nuisance through camera jitter, a simplified proxy. The diagnostic
requires a small labeled subsample and is most reliable in two-stage pipelines.

\begin{ack}
IK is supported by the Azrieli Graduate Studies Fellowship. 
AR is an employee of Genentech and has equity in Roche.
HV is an employee of Genentech
\end{ack}

\bibliography{main}
\newpage

\appendix

\section{Closed-form solutions and spiked model derivations}
\label{app:closed_form}

Both results are classical: \Cref{thm:closed_form_ca} is the CCA solution
\citep{Hotelling1936_CCA}, and \Cref{thm:closed_form_cp} is
truncated reduced-rank regression \citep{izenman1975}. We restate them in our
notation, with proofs, because the spiked-model analysis builds on their explicit form.

\subsection{Full statement of closed-form solutions}
\label{app:closed_form_statement}

\begin{theorem}[Closed-form solutions for CA]
\label{thm:closed_form_ca}
Assume $\mSxx$ and $\mSyy$ are positive definite, and let
$k \leq \min(d_x, d_y)$.
Let $\mC = \mP \mPhi \mQ^\top$ be the SVD of
$\mC \coloneqq \mSxx^{-1/2} \mSxy \mSyy^{-1/2}$ with singular
values $\phi_1 \geq \cdots \geq \phi_{\min(d_x,d_y)} \geq 0$.
The minimizers of \eqref{eq:ca} with linear encoders are
\begin{equation}
\mW^\star = \mU \mP_k^\top \mSxx^{-1/2},
\quad
\mV^\star = \mU \mQ_k^\top \mSyy^{-1/2},
\label{eq:ca_linear_solution}
\end{equation}
where $\mP_k, \mQ_k$ contain the leading $k$ columns and
$\mU \in \R^{k \times k}$ is an arbitrary orthogonal matrix. 
\end{theorem}

\begin{theorem}[Closed-form solutions for CP]
\label{thm:closed_form_cp}
Let $\mA = \mU_{\!A} \diag(\boldsymbol{\sigma}) \mV_{\!A}^\top$ be the SVD of
$\mA := \mSyx \mSxx^{-1/2}$ with $\sigma_1 \geq \cdots \geq \sigma_r > 0$.
The composed map $\mB := \mD \mE$ at a minimizer of \eqref{eq:cp} with
linear encoder and decoder is
\begin{equation}
\mB^\star = \mU_{\!A,k} \, \diag(\sigma_1, \dots, \sigma_k) \, \mV_{\!A,k}^\top \, \mSxx^{-1/2}.
\label{eq:cp_linear_solution}
\end{equation}

The factorization $\mB^\star = \mD^\star \mE^\star$ is non-unique:
for any invertible $\mM \in \R^{k \times k}$, $(\mD^\star \mM, \mM^{-1} \mE^\star)$
yields the same composed map.
\end{theorem}

\subsection{Proof of \Cref{thm:closed_form_ca}}\label{app:proof_ca}
For linear encoders the constraint reads
$\mW \mSxx \mW^\top = \mV \mSyy \mV^\top = \mI_k$, so
$\tfrac{1}{n}\sum_i \|\mW \vx_i\|^2 = \tfrac{1}{n}\sum_i \|\mV \vy_i\|^2 = k$
and the objective reduces to
\begin{align}
    \min_{\mW,\mV} \quad 2k - 2\tr{(\mW \mSxy \mV^\top)}
    \quad \text{s.t.} \quad
    \mW \mSxx \mW^\top = \mV \mSyy \mV^\top = \mI_k \:,
\end{align}
that is, to maximizing $\tr{(\mW \mSxy \mV^\top)}$.
Let $\mW' = \mW \mSxx^{1/2}$ and $\mV' = \mV \mSyy^{1/2}$, and define
$\mC \coloneqq \mSxx^{-1/2} \mSxy \mSyy^{-1/2}$. The constraints
become $\mW' \mW'^\top = \mV' \mV'^\top = \mI_k$, so $\mW'$ and $\mV'$ have
orthonormal rows, and the problem becomes
\begin{align}
    \max_{\mW',\mV'} \quad \tr{(\mW' \mC \mV'^\top)}
    = \langle \mC, \mW'^\top \mV' \rangle
    \quad \text{s.t.} \quad
    \mW' \mW'^\top = \mV' \mV'^\top = \mI_k \:.
\end{align}
The matrix $\mN \coloneqq \mW'^\top \mV'$ is a product of two partial
isometries, hence $\mathrm{rank}(\mN) \leq k$ and all singular values of
$\mN$ are at most $1$. By von Neumann's trace inequality \citep{mirsky_trace_1975},
\begin{align}
    \langle \mC, \mN \rangle
    \;\leq\; \sum_{i \leq k} \phi_i \, \sigma_i(\mN)
    \;\leq\; \sum_{i \leq k} \phi_i \:.
\end{align}
Equality is attained by $\mW' = \mU \mP_k^\top$ and $\mV' = \mU \mQ_k^\top$
with $\mU$ orthogonal, which satisfy both constraints and give
$\mN = \mP_k \mQ_k^\top$. Transforming back gives
\begin{align}
    \mW^\star = \mU \mP_k^\top \mSxx^{-1/2} \:, \qquad
    \mV^\star = \mU \mQ_k^\top \mSyy^{-1/2} \:.
\end{align}

\subsection{Proof of \Cref{thm:closed_form_cp}}\label{app:proof_cp}
Since $\mE \in \R^{k \times d_x}$ and $\mD \in \R^{d_y \times k}$, the composed
map $\mB = \mD\mE$ satisfies $\mathrm{rank}(\mB) \leq k$.
Conversely, any rank-$k$ matrix $\mB$ admits such a factorization, so
minimizing over $(\mD,\mE)$ is equivalent to minimizing over rank-$k$ matrices
$\mB$.
Writing the CP objective in terms of $\mB$ and expanding:
\begin{align}
    \frac{1}{n}\sum_{i} \|\vy_i - \mB\vx_i\|^2
    &= \tr(\mSyy) - 2\,\tr(\mSyx\mB^\top) + \tr(\mB\mSxx\mB^\top)\:.
\end{align}
Substituting $\mB' \coloneqq \mB\mSxx^{\tfrac{1}{2}}$ and
$\mA \coloneqq \mSyx\mSxx^{-\tfrac{1}{2}}$ (so that
$\mSyx\mB^\top = \mA(\mB')^\top$ and
$\mB\mSxx\mB^\top = \mB'(\mB')^\top$):
\begin{align}
    &= \tr(\mSyy) - 2\,\tr(\mA(\mB')^\top) + \tr(\mB'(\mB')^\top) \\
    &= \tr(\mSyy) - \tr(\mA\mA^\top)
       + \|\mB' - \mA\|_F^2 \:.
\end{align}
Since $\tr(\mSyy) - \tr(\mA\mA^\top)$ does not depend on $\mB$, and since
$\mSxx^{\tfrac{1}{2}}$ is invertible, minimizing over rank-$\leq k$ matrices
$\mB$ is equivalent to minimizing $\|\mB' - \mA\|_F^2$ over rank-$\leq k$
matrices $\mB'$.
By the Eckart--Young--Mirsky theorem~\citep{eckart_approximation_1936, Mirsky1960Q}, the best rank-$k$ approximation of $\mA$
in Frobenius norm is
$(\mB')^\star = \mU_{\!A,k} \, \diag(\sigma_1, \dots, \sigma_k) \, \mV_{\!A,k}^\top$.
Transforming back via $\mB^\star = (\mB')^\star \mSxx^{-\tfrac{1}{2}}$ gives
the stated solution.
\paragraph{Non-uniqueness of the factorization.}
The product $\mB^\star$ is uniquely determined (assuming distinct singular values),
but the factorization $\mB^\star = \mD^\star\mE^\star$ is not: for any invertible
$\mM \in \R^{k\times k}$, the pair $(\mD^\star\mM,\, \mM^{-1}\mE^\star)$ yields
the same composed map. Hence $\mE^\star$ is determined only up to left-multiplication
by an invertible matrix.

\subsection{Full parameterization}

For simplicity, assume $d_x = d_y = d$ and that there exist orthogonal
matrices $\mathbf{Q}_x, \mathbf{Q}_y$ such that
\begin{align}
\mSxx &= \mathbf{Q}_x \Lambda_x \mathbf{Q}_x^\top,
& \Lambda_x &= \diag\!\bigl(\mK^2 + \mGamma_x^{(s)},\; \mGamma_x^{(n)}\bigr), \\
\mSyy &= \mathbf{Q}_y \Lambda_y \mathbf{Q}_y^\top,
& \Lambda_y &= \diag\!\bigl(\mK^2 + \mGamma_y^{(s)},\; \mGamma_y^{(n)}\bigr), \\
\mSxy &= \mathbf{Q}_x \Lambda_{xy} \mathbf{Q}_y^\top,
& \Lambda_{xy} &= \diag\!\bigl(\mK^2,\; \mGamma_{xy}\bigr),
\end{align}
where $\mK = \diag(\kappa_1, \ldots, \kappa_k)$ with $\kappa_1 \geq \cdots \geq \kappa_k > 0$,
$\mGamma_x^{(s)} = \diag(\gamma_1^x, \ldots, \gamma_k^x)$,
$\mGamma_y^{(s)} = \diag(\gamma_1^y, \ldots, \gamma_k^y)$,
$\mGamma_x^{(n)} = \diag(\tilde\gamma_1^x, \ldots, \tilde\gamma_{d-k}^x)$,
$\mGamma_y^{(n)} = \diag(\tilde\gamma_1^y, \ldots, \tilde\gamma_{d-k}^y)$,
and $\mGamma_{xy} = \diag(\eta_1, \ldots, \eta_{d-k})$ with
$0 \leq \eta_j \leq \sqrt{\tilde\gamma_j^x \tilde\gamma_j^y}$.

\subsection{Singular-value decompositions}
\label{app:singular_values}

\begin{lemma}[Singular values of $\mC$ and $\mA$]
\label{lem:singular_values}
Under the spiked model, $\mC$ and $\mA$ are block diagonal in the bases
defined by $\mathbf{Q}_x, \mathbf{Q}_y$. Their singular values are the union
of the signal values $\rho_i, \tau_i$ and nuisance values $\nu_j, \xi_j$
given in \eqref{eq:singular_values}.
\end{lemma}

\emph{Proof.} In the bases $\mathbf{Q}_x, \mathbf{Q}_y$, both
$\mC = \mSxx^{-1/2} \mSxy \mSyy^{-1/2}$ and $\mA = \mSyx \mSxx^{-1/2}$
are block diagonal with signal and nuisance blocks. Direct computation on
each block yields the stated expressions.

\begin{corollary}[Recovery conditions]
If $\min_i \rho_i > \max_j \nu_j$, the top-$k$ singular vectors of $\mC$
align with the shared signal block, so the CA solution from \Cref{thm:closed_form_ca} recovers the shared signal subspace (up to
rotation) and discards modality-specific noise. The analogous statement
holds for $\mA$ and CP with $\tau_i, \xi_j$.
\end{corollary}

\subsection{Proof of \Cref{prop:separation}}
\label{app:separation_proof}
The singular-value expressions follow from \Cref{lem:singular_values}.
For the ratio identity, we consider the homogeneous case
($\kappa_i \equiv \kappa$, $\gamma_i^x \equiv \gamma^x$,
$\gamma_i^y \equiv \gamma^y$, $\tilde\gamma_j^x \equiv \tilde\gamma^x$,
$\tilde\gamma_j^y \equiv \tilde\gamma^y$, $\eta_j \equiv \eta$),
in which all $\rho_i$ collapse to a single value $\rho$, and likewise
$\tau_i \equiv \tau$, $\nu_j \equiv \nu$, $\xi_j \equiv \xi$.
Substituting \eqref{eq:singular_values} into $\DCA/\DCP = (\rho/\nu)/(\tau/\xi)$ yields
\begin{equation}
\frac{\DCA}{\DCP}
= \sqrt{\frac{\tilde\gamma^y}{\kappa^2 + \gamma^y}}.
\end{equation}
In the heterogeneous case, the same substitution yields the upper bound
\begin{equation}
\frac{\DCA}{\DCP}
\leq \sqrt{\frac{\max_j \tilde\gamma_j^y}{\min_i(\kappa_i^2 + \gamma_i^y)}},
\end{equation}
since the indices achieving $\min_i \rho_i$ and $\min_i \tau_i$ (and similarly
the nuisance maxima) need not coincide. Monotonicity of $\DCA/\DCP$ in each
$\tilde\gamma_j^y$ follows from $\DCP$ being invariant in $\tilde\gamma_j^y$
(since $\xi_j = \eta_j/\sqrt{\tilde\gamma_j^x}$ does not depend on target
nuisance variance) and $\DCA$ being non-increasing in $\nu_j$, with $\nu_j$
non-increasing in $\tilde\gamma_j^y$.

\subsection{Proof of \Cref{prop:partial_recovery}}
\label{app:partial_recovery_proof}

We prove the CA case; the CP case is identical with $\rho_i, \nu_j$
replaced by $\tau_i, \xi_j$ and $\mC$ replaced by $\mA$.

By \Cref{lem:singular_values}, the singular values of $\mC$ are the union
$\{\rho_i\}_{i \in \llbracket k \rrbracket} \cup \{\nu_j\}_{j \in \llbracket d-k \rrbracket}$,
with each $\rho_i$ corresponding to a singular vector in the signal block
and each $\nu_j$ to a singular vector in the nuisance block.

Let $i$ be any signal index with $\rho_i > \max_j \nu_j$. Then $\rho_i$
exceeds every nuisance singular value, so at most $k - 1$ values of $\mC$
(the other signal values) can exceed $\rho_i$, placing $\rho_i$ among the
top $k$ singular values. Since this holds for each of the $r_{\mathrm{CA}}$
signal indices with $\rho_i > \max_j \nu_j$, the top-$k$ singular vectors
contain at least $r_{\mathrm{CA}}$ vectors from the signal block.
\clearpage
\section{Implementation Details}\label{sec:implementation_details}

\paragraph{Linear experiments.}
All linear experiments use closed-form solvers on population covariances drawn from the spiked model of \Cref{sec:solutions_and_bottlenecks}. We set $d=20$, $k=3$ shared dimensions with signal strengths $\kappa = (3.0, 2.0, 1.5)$, and consider a regime with clean signal in the target modality ($\gamma_y^{(s)} = 0.05$) but large target nuisance variance ($\gamma_y^{(n)} = 50.0$), with source-side parameters $\gamma_x^{(s)} = 0.5$, $\gamma_x^{(n)} = 1.0$. We sweep the normalized noise correlation $\nu = \eta / \sqrt{\gamma_x^{(n)} \gamma_y^{(n)}} \in [0, 0.95]$ and report subspace distance (averaged over 20 random rotations) and theoretical separation ratios. Subspace distances are computed as $\|\mP_{\hat{U}} - \mP_{U}\|_F / \sqrt{2k}$ where $\mP$ denotes the orthogonal projector, and averaged over 20 random rotations of the signal/noise bases. Here $\mathbf{P}_U$ is the orthogonal projector onto the column span of $U$, so the distance
is basis-invariant, lies in $[0, 1]$, and applies equally to CA and CP, whose recovered
subspaces both live in $\R^{d_x}$. No optimization is involved; the solvers compute exact CA (CCA) and CP (truncated reduced-rank regression) solutions from the covariance matrices on a single CPU.

\paragraph{Stereo-dSprites.}
Two virtual cameras observe a shared 2D object (Square, Ellipse, or Heart) on a $64 \times 64$ grayscale canvas. World position $P_{\text{world}} \in [-0.5, 0.5]^2$ serves as the aligned nuisance; camera jitter $\sigma_{\text{jitter}} \in \{0.0, 0.05, 0.2, 0.5\}$ controls de-alignment via per-view translation and rotation. View~X receives Gaussian pixel noise $\sigma_{\text{strong}} = 0.1$; View~Y receives $\sigma_{\text{weak}} \in \{0.2, 0.5, 0.9\}$. Each modality is encoded by a separate 4-layer CNN ($1 \times 64 \times 64 \to 128$-dim) with ReLU activations. CA uses VICReg ($25 \times$ invariance $+ 25 \times$ variance $+ 1 \times$ covariance) on 32-dimensional projections. CP uses MSE reconstruction via a transposed-convolutional decoder. All models are trained with Adam (lr=$10^{-3}$), batch size 64, for up to 100 epochs with early stopping (patience 5). Downstream evaluation uses a linear probe on frozen encoder features for 3-class shape classification, swept over 9 probe sizes (100 to ${\sim}6{,}000$ samples) and averaged over 5 seeds. We sweep $n_{\text{samples}} \in \{10\text{k}, 50\text{k}, 100\text{k}\}$ (see \Cref{fig:dsprites_samples} for $10\text{k}, 50\text{k}$ results). The entire sweep took approximately 24 hours on one L40S GPU. 

\paragraph{Stereo-3DShapes.}
Built from Google's 3DShapes dataset (480K RGB images). Canonical images, one per shape at fixed hue, scale, and orientation, are rendered into stereo pairs via affine warping with the same jitter and noise protocol as dSprites. The signal is 4-class shape (Cube, Cylinder, Sphere, Capsule); additional nuisance factors include floor hue (10 values), wall hue (10), object hue (10), scale (8), and orientation (15), all fixed per canonical image and shared across views. Encoders are 4-layer CNNs ($3 \times 64 \times 64 \to 128$-dim; channels $3\to 32\to 32\to 64\to 64$, stride 2) with FC layers $1024 \to 256 \to 128$. Training and evaluation follow the dSprites protocol with $n_{\text{samples}} = 100\text{k}$, 10 probe sizes (100 to 10{,}000), and 3--4 seeds. The entire sweep took approximately 24 hours on one L40S GPU.

\paragraph{MS-COCO image-caption.}
\label{par:coco_implementation}
We pair each COCO~2017 image with its associated caption, using the
dominant-object category (largest bounding box, 80 classes) as the
downstream label. Images are $3 \times 224 \times 224$ RGB, encoded by a
ResNet-18 trained from scratch; captions are word-tokenized to length
$64$ and encoded by a 2-layer Transformer (4 heads, $d_{\text{embed}} = 256$)
followed by mean pooling and a linear projection to $128$ dimensions.
Neither encoder uses pretrained weights.  Nuisance is injected into the image modality: each image is passed through $k$ independent distortion groups (color cast, exposure,
contrast, texture, saturation, spatial) drawn uniformly from six groups, with $k$ controlled by a noise level
$\ell \in \{0.0, 0.2, 0.5\}$ (expected $k \approx 6\ell$ groups active).
Within each active group, a random transform is applied at continuous
intensity $t \sim \mathrm{Uniform}(0.3, 1.0)$, so every sample receives a unique pixel-level distortion. Training uses
AdamW (lr = $10^{-3}$, weight-decay = $10^{-4}$) with 5-epoch linear warmup into cosine annealing, batch size $1024$ per GPU across 4--6 GPUs
via DDP, for $50$ epochs. CA uses VICReg on projected embeddings from
both encoders (invariance $25\times$, variance $25\times$, covariance
$1\times$); CP$_{I \to T}$ is the image encoder feeding a caption decoder
under cross-entropy loss; CP$_{T \to I}$ is the text encoder feeding a
pixel decoder under MSE. Evaluation: each method is evaluated on its source /bottleneck encoder. 
For CA, the image encoder is probed (a symmetric choice — both encoders 
are equally optimized under VICReg). For $\mathrm{CP}_{I \to T}$, the 
image encoder (source). For $\mathrm{CP}_{T \to I}$, the text encoder 
(source). Each frozen representation is fed to a linear probe trained for 
30 epochs on the 80-class label; we report
top-1 accuracy. The experiment took approximately 24 hours on 4 RTX6000 GPUs.

\paragraph{Astrophysical cross-modal.}
\label{par:astro_implementation}
We pair LAMOST optical spectra (DR8, resolution $R \sim 1800$, range $3690$--$9100$\,\AA) with light curves from two photometric surveys —
Kepler (DR25, 30-min cadence, $\sim\!4$\,year baseline; $94{,}876$ cross-matched observations) and TESS (QLP lightcurves;
$821{,}878$ cross-matched stars). Each modality uses its own pretrained unimodal encoder, frozen during cross-modal training. LAMOST spectra are
encoded by a '1d ViT' that produces a $2048$-dim CLS token; light curves are encoded by a multichannel network with parallel flux and frequency (ACF and
FFT/Lomb–Scargle) branches combined through a mixer, producing a
$1024$-dim mean-pooled embedding. Both encoders were pretrained
independently on their respective modalities before any cross-modal
training. On these frozen features we train lightweight heads for each
method: CA uses two projection MLPs ($2048 \to 512$ and $1024 \to 512$)
with a VICReg objective (invariance $25\times$, variance $25\times$,
covariance $1\times$); CP uses a cross-predictor MLP
with a $512$-dim bottleneck and MSE loss. Optimization: AdamW
(lr\,=\,$10^{-3}$, weight decay\,=\,$10^{-4}$) with cosine annealing,
batch size $256$, early stopping on validation loss. Evaluation uses a
linear probe on the cross-modal representation (concatenated projections
for CA, bottleneck activations for CP) against held-out stellar labels —
$\log g$, age (regression), and binarity (classification). The cross-modal experiment took approximately 4 hours on one RTX6000 GPU. 

\paragraph{BMMC CITE-seq.}
\label{par:bmmc_implementation}
We pair RNA expression with surface-protein abundance measured in the same cells, using
the BMMC CITE-seq dataset ($90{,}243$ bone-marrow mononuclear cells~\citep{Luecken2021ASF}), with $44$ annotated cell types as the downstream label.
RNA is encoded by scGPT \citep{cui_scgpt_2024} in a one-time pass over the $2{,}000$
highly variable genes, giving cached $512$-dim embeddings that are never updated. The
protein modality enters as the $134$ CLR-normalized values shipped with the dataset,
passed through a small MLP ($134 \to 256 \to 512 \to 512$) trained jointly with the
cross-modal objective; a separate encoder is trained per mode, so the ADT-only baseline is
mode-specific and we report the CA-mode value. CA uses projection MLPs onto $256$
dimensions per view with a VICReg objective (invariance $25\times$, variance $25\times$,
covariance $1\times$); CP uses a cross-predictor MLP with a $512$-dim bottleneck and MSE
loss, in both directions. Optimization: AdamW (lr\,=\,$10^{-3}$, weight
decay\,=\,$10^{-4}$) with cosine annealing, batch size $512$, $200$ epochs, early stopping
on validation loss (patience $30$). Splits are donor-based ($6$ train / $1$ validation /
$2$ test donors), averaged over $5$ seeds with a different donor assignment per seed.
Evaluation uses a linear probe (standardization, class-balanced logistic regression) on
the cross-modal representation, concatenated projections for CA, bottleneck activations
for CP, reported as balanced accuracy. Regime estimation
(\Cref{alg:phase_estimation}) runs on a seeded subsample of $n = 5{,}000$ cells, on the
same representations the objective consumes: scGPT embeddings and raw CLR counts, each
centered and reduced to $100$ principal components, $\varepsilon = 10^{-6}$; PCA is used
only for estimation. Across the five seeds, $\hat r_{\mathrm{CA}} = 13.0 \pm 3.3$ of
$\hat k_{\mathrm{CA}} = 20.4 \pm 0.5$ and $\hat r_{\mathrm{CP}} = 12.6 \pm 3.2$ of
$\hat k_{\mathrm{CP}} = 22.4 \pm 1.0$ signal components clear the nuisance floor. Training all three objectives and their probes took approximately
$1.5$ hours per seed on one L40S GPU.

We keep the protein branch single-stage because, with $134$ measured proteins, CLR counts
are already a compact representation. As a check, we repeated the experiment with the ADT
encoder initialized from a unimodal VICReg pretraining stage (same architecture, trained on
ADT alone with two augmented views of each CLR vector via Gaussian noise
$\sigma \sim \mathcal{U}(0, 0.1)$ and random feature dropout $p \sim \mathcal{U}(0, 0.1)$;
projection heads discarded), and estimated the regime from that encoder's frozen
$512$-dim embeddings rather than raw CLR counts. Both ratios decrease
($\hat\Delta_{\mathrm{CA}} = 0.34$, $\hat\Delta_{\mathrm{CP}} = 0.11$ versus $0.67$ and
$0.26$ at the canonical seed), as expected from a decorrelating unimodal objective, but the
regime (Neither in $5/5$ seeds), the ordering
$\mathrm{CA} > \mathrm{CP} > \mathrm{CP}_{\mathrm{rev}}$, and the downstream accuracies are
unchanged.

\begin{algorithm}[h]
\caption{Recovery Regime Prediction}
\label{alg:phase_estimation}
\begin{algorithmic}[1]
\REQUIRE Paired embeddings $\mZ_x \in \mathbb{R}^{n \times d_x}$, $\mZ_y \in \mathbb{R}^{n \times d_y}$; labels $\mY \in \mathbb{R}^{n \times L}$
\ENSURE $\hat{\Delta}_{\mathrm{CA}}$, $\hat{\Delta}_{\mathrm{CP}}$, predicted regime
\STATE \COMMENT{CA: read separation ratio from CCA spectrum}
\STATE Compute SVD of $\hat{\mC} = (\hat{\mSigma}_{xx} + \varepsilon \mI)^{-1/2} \hat{\mSigma}_{xy} (\hat{\mSigma}_{yy} + \varepsilon \mI)^{-1/2}$; obtain $\phi^{\mathrm{CCA}}$
\STATE Classify CCA components as signal/nuisance by elbow detection on per-component $R^2$
\STATE $\hat{\Delta}_{\mathrm{CA}} \leftarrow \min(\phi^{\mathrm{CCA}}[\mathrm{signal}]) \,/\, \max(\phi^{\mathrm{CCA}}[\mathrm{nuisance}])$
\STATE \COMMENT{CP: read separation ratio from $\mA$-SVD spectrum}
\STATE Compute SVD of $\hat{\mA} = \hat{\mSigma}_{yx} (\hat{\mSigma}_{xx} + \varepsilon \mI)^{-1/2}$; obtain $\sigma^{\mA}$
\STATE Classify $\mA$-SVD components as signal/nuisance by elbow detection on per-component $R^2$
\STATE $\hat{\Delta}_{\mathrm{CP}} \leftarrow \min(\sigma^{\mA}[\mathrm{signal}]) \,/\, \max(\sigma^{\mA}[\mathrm{nuisance}])$
\RETURN $(\hat{\Delta}_{\mathrm{CA}}, \hat{\Delta}_{\mathrm{CP}})$ and regime per $\gtrless 1$
\end{algorithmic}
\end{algorithm}
\FloatBarrier

\paragraph{Recovery regime prediction.}
\label{par:phase_estimation_details}
The pipeline takes paired embeddings
$\mZ_x \in \mathbb{R}^{n \times d_x}$,
$\mZ_y \in \mathbb{R}^{n \times d_y}$ and labels
$\mY \in \mathbb{R}^{n \times L}$ for a labeled subsample.
Under the spiked model, $\hat\Delta_\mathrm{CA}$ and $\hat\Delta_\mathrm{CP}$
are singular-value ratios of $\mC$ and $\mA$ respectively: signal and
nuisance singular values appear together in each spectrum, distinguished
only by which block of the spiked decomposition they belong to. We
estimate each $\Delta$ directly from the corresponding spectrum,
classifying each component as signal or nuisance by its predictive power
for the labels. For each decomposition, 5-fold Ridge regression of the
component scores against the labels yields a per-component $R^2$ value
(summed across label columns for the classification statistic), and
piecewise-linear elbow detection on the sorted $R^2$ curve identifies
the breakpoint between the signal block and the nuisance floor. The
labeled subsample need not cover the full training set: in our
experiments, $n < 1{,}000$ samples suffice.

\clearpage
\section{Additional analyses}
\label{app:additional}

\subsection{InfoNCE in the linear setting}
\label{app:infonce}

CLIP-style alignment relies on InfoNCE with many negatives rather than an explicit whitening
constraint, which raises the question of whether negatives resist cross-modal nuisance
correlation. We add an InfoNCE-trained linear encoder (symmetric NT-Xent, linear
$\mathbf{W}, \mathbf{V}$, Adam with learning rate $10^{-2}$, batch size $1024$) to the
closed-form CA and CP solvers of \Cref{fig:linear_E4}, in the same spiked model
($d = 20$, $k = 3$, $\kappa = (3.0, 2.0, 1.5)$, $\gamma_x^{(s)} = 0.5$,
$\gamma_y^{(s)} = 0.05$, $\gamma_x^{(n)} = 1$, $\gamma_y^{(n)} = 50$, $n = 5{,}000$), and
sweep $\nu$ over $15$ points in $[0, 0.95]$ with $20$ seeds per point and temperatures
$\tau \in \{0.05, 0.1, 0.5\}$.

InfoNCE is somewhat worse than CA throughout but follows the same failure curve and remains
far from CP (\Cref{tab:infonce}): taking failure as the first grid point where the mean
distance exceeds $0.5$, CP fails at $\nu \approx 0.14$, InfoNCE at $\nu \approx 0.88$, and CA
at $\nu \approx 0.95$. Temperature has negligible effect: across all $15$ points the other
two temperatures differ from $\tau = 0.1$ by at most $0.014$, below the seed standard
deviation. Adding negatives therefore does not rescue alignment when nuisance is strongly
cross-correlated. The reason is structural: negatives estimate the product of marginals in
the InfoNCE denominator and down-weight structure shared across the batch, but correlated
nuisance is a property of each matched pair and is independent across pairs, so the
negatives cannot distinguish it from signal.

\begin{table}[h]
\centering
\caption{Subspace distance to the true signal subspace (lower is better), mean over $20$
seeds, at six of the $15$ swept values of $\nu$. InfoNCE at $\tau = 0.1$.}
\label{tab:infonce}
\small
\begin{tabular}{lcccccc}
\toprule
$\nu$ & $0.14$ & $0.41$ & $0.68$ & $0.81$ & $0.88$ & $0.95$ \\
\midrule
CP (RRR)  & $0.57$ & $1.00$ & $1.00$ & $1.00$ & $1.00$ & $1.00$ \\
CA (CCA)  & $0.05$ & $0.05$ & $0.08$ & $0.14$ & $0.50$ & $0.84$ \\
InfoNCE   & $0.11$ & $0.13$ & $0.21$ & $0.40$ & $0.73$ & $0.97$ \\
\bottomrule
\end{tabular}
\end{table}

\subsection{Stability under subsampling}
\label{app:subsample}

For each pair and labeled-subsample size $m$, we draw $50$ subsamples without replacement,
re-run the full estimator (PCA to $100$ components after subsampling,
$\varepsilon = 10^{-6}$), and record the fraction of draws in each regime together with the
median and interquartile range of the recovery counts (\Cref{tab:subsample}). At every size, most draws reproduce the full-sample regime, and the agreement grows with $m$:
TESS and BMMC are Neither in every draw at the largest $m$, and LAMOST$\times$Kepler, whose
$\hat\Delta_{\mathrm{CA}} = 1.08$ sits close to the CA boundary, is Both in $58\%$ of draws
at $m = 250$ rising to $80\%$ at $m = 750$. The recovery count is markedly more stable than
the binary label: its median and interquartile range are $0$ in every TESS cell, positive in
every BMMC cell, and $\hat r_{\mathrm{CA}} = 1$ in every Kepler cell. The distinction
between complete failure and partial recovery is therefore preserved at every subsample
size, including where the regime label fluctuates.

\begin{table}[h]
\centering
\caption{Subsampling stability, $50$ draws without replacement per cell,
$\varepsilon = 10^{-6}$. Regime fractions list all regimes with nonzero share; recovery
counts are median [IQR]. Full sample sizes: Kepler $934$, TESS $153$, BMMC $90{,}243$.}
\label{tab:subsample}
\small
\begin{tabular}{llllll}
\toprule
Pair & $m$ & Regime fractions & $\hat r_{\mathrm{CA}}$ & $\hat r_{\mathrm{CP}}$ \\
\midrule
Kepler & $250$ & Both $.58$, CP $.24$, Neither $.10$, CA $.08$ & $1\,[1,1]$ & $2\,[1,2]$ \\
       & $500$ & Both $.78$, CP $.18$, Neither $.04$ & $1\,[1,1]$ & $2\,[2,2]$ \\
       & $750$ & Both $.80$, CP $.20$ & $1\,[1,1]$ & $2\,[2,2]$ \\
\midrule
TESS   & $80$  & Neither $.98$, CA $.02$ & $0\,[0,0]$ & $0\,[0,0]$ \\
       & $110$ & Neither $.94$, CA $.06$ & $0\,[0,0]$ & $0\,[0,0]$ \\
       & $140$ & Neither $1.00$ & $0\,[0,0]$ & $0\,[0,0]$ \\
\midrule
BMMC   & $250$    & Neither $.78$, CA $.14$, CP $.08$ & $6\,[5,7]$    & $9\,[6,11]$ \\
       & $500$    & Neither $.92$, CA $.06$, CP $.02$ & $11\,[10,12]$ & $12\,[10,14]$ \\
       & $1{,}000$  & Neither $.82$, CA $.12$, CP $.06$ & $14\,[11,17]$ & $14\,[12,17]$ \\
       & $5{,}000$  & Neither $.98$, CP $.02$ & $17\,[11,17]$ & $11\,[11,11]$ \\
       & $72{,}000$ & Neither $1.00$ & $17\,[17,17]$ & $11\,[11,11]$ \\
\bottomrule
\end{tabular}
\end{table}

\subsection{Sensitivity to label quality}
\label{app:label_robustness}

We degrade the label matrix while holding the estimator fixed (\Cref{tab:label_robustness}).

\emph{Label noise.} Adding Gaussian noise to the $\log g$ label, calibrated so that its
feature-$R^2$ falls to $0.75$, $0.50$, and $0.25$ of the clean value, leaves both astro calls
unchanged: Kepler stays Both with $\hat r_{\mathrm{CA}} = 1$, TESS stays Neither.

\emph{Label corruption and coarsening.} Randomly reassigning up to $50\%$ of BMMC cell-type
labels, or coarsening the $44$ cell types to $8$ lineages, leaves BMMC in Neither with
$\hat r > 0$; reassigning $25\%$ or more reduces $\hat r_{\mathrm{CA}}$ from $17$ to $10$ and
$\hat\Delta_{\mathrm{CA}}$ from $0.67$ to $0.51$. Quantizing the continuous $\log g$ label
into $8$ or $4$ bins moves Kepler from Both to CP-only, but not by losing signal: it adds a
second, weak component to the CA signal set ($\hat k_{\mathrm{CA}}$ from $1$ to $2$) that
sits below the nuisance floor, pulling $\hat\Delta_{\mathrm{CA}}$ below one while
$\hat r_{\mathrm{CA}}$ stays $1$ --- another instance of the recovery count being the
more stable readout.

\emph{Unsupervised proxy labels.} Replacing the labels by $k$-means clusters of the weaker
modality's raw features ($k \in \{8, 20, 44\}$) does not reproduce the ground truth. On
TESS, which has no recoverable shared signal, every proxy produces a false positive
($\hat r_{\mathrm{CA}} \in \{1, 2, 3\}$), two of them with a CP-only regime call; every proxy
also miscalls Kepler, and $k = 8$ yields a spurious CP-only call on BMMC. Clustering one
modality tracks that modality's dominant variance, which in the complementary regime is
nuisance rather than shared signal, so the labeled subsample is a genuine requirement of the
diagnostic.

\begin{table}[h]
\centering
\caption{Regime and recovery counts $(\hat r_{\mathrm{CA}}, \hat r_{\mathrm{CP}})$ under
label degradation, $\varepsilon = 10^{-6}$. Ground truth: Kepler Both $(1, 2)$, TESS
Neither $(0, 0)$, BMMC Neither $(17, 11)$. Proxy rows use $k$-means on the weaker
modality's raw features.}
\label{tab:label_robustness}
\small
\begin{tabular}{llll}
\toprule
Condition & Kepler & TESS & BMMC \\
\midrule
Noise, $R^2 \to 0.25\times$  & Both $(1,2)$    & Neither $(0,0)$ & --- \\
Reassign $50\%$              & ---             & ---             & Neither $(10,11)$ \\
$44 \to 8$ lineages          & ---             & ---             & Neither $(4,5)$ \\
$\log g$ in $8$ bins         & CP-only $(1,2)$ & Neither $(0,0)$ & --- \\
\midrule
$k$-means, $k = 8$           & CP-only $(1,3)$ & CP-only $(1,4)$ & CP-only $(5,8)$ \\
$k$-means, $k = 20$          & Neither $(5,2)$ & CP-only $(2,2)$ & Neither $(12,10)$ \\
$k$-means, $k = 44$          & Neither $(5,2)$ & Neither $(3,2)$ & Neither $(14,17)$ \\
\bottomrule
\end{tabular}
\end{table}

\subsection{Origin of the BMMC nuisance floor}
\label{app:bmmc_floor}

$\hat\Delta_{\mathrm{CA}}$ is a ratio whose denominator is the largest nuisance canonical
correlation, so a single component sets the nuisance floor. To identify it, we regress the
floor-setting component's score on cell-level metadata (canonical seed, $n = 5{,}000$). It is
explained by technical rather than biological covariates: batch ($R^2 = 0.35$), site ($0.25$),
and donor ($0.21$), with sequencing depth, mitochondrial fraction, cell-cycle phase, and
lineage all near zero. The floor component of the $\mathbf{A}$-SVD is likewise batch-loaded
($R^2 = 0.25$).

We then ask whether removing this structure lifts the pair out of Neither. Regressing batch
and sequencing depth out of both modalities lowers $\hat\Delta_{\mathrm{CA}}$ from $0.67$ to
$0.52$, and centering each modality within donor gives $0.61$; both remain Neither. In the
first case the floor moves to a different component that is itself batch-loaded: removing the
dominant batch direction promotes the next. The floor is thus a multi-dimensional technical
subspace rather than a single removable direction, which is why linear batch correction does
not move the pair across the boundary.

\section{Additional Figures}\label{sec:additional_figs}

\begin{figure}[t]
    \centering
     \begin{minipage}{0.58\linewidth}
    \includegraphics[width=\linewidth]{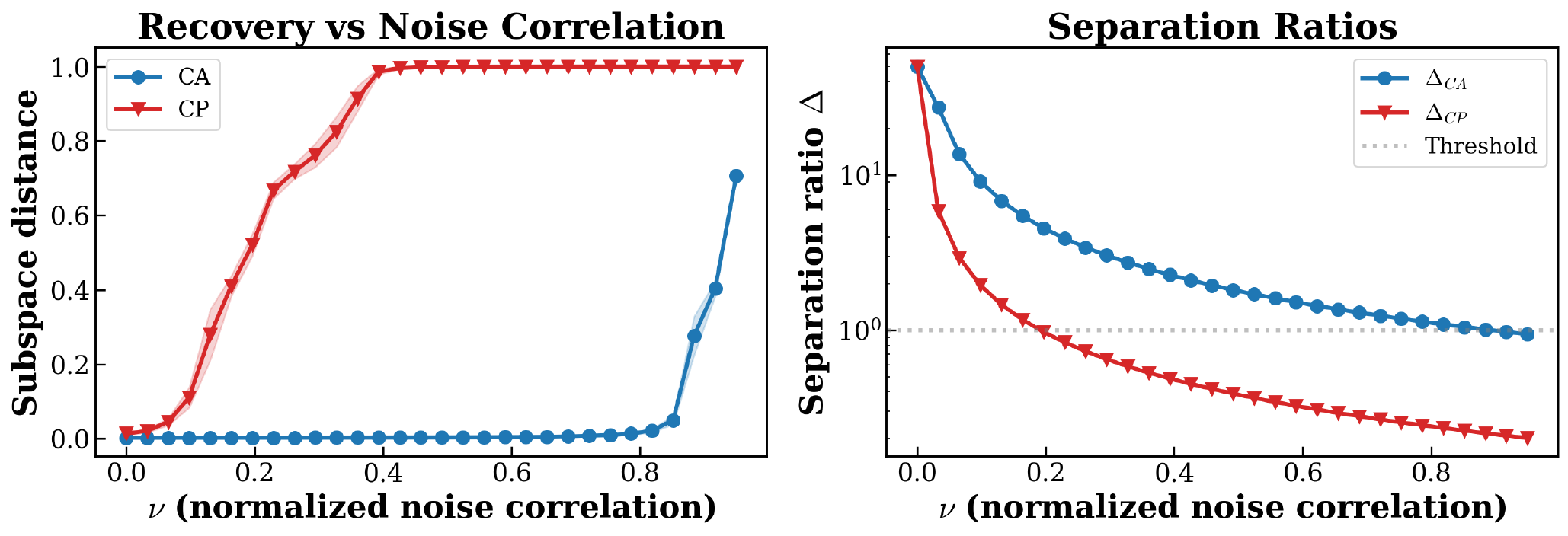}
    \end{minipage}\hfill
    \begin{minipage}{0.41\linewidth}
    \caption{Recovery as a function of normalized noise correlation $\nu$. \figleft~Subspace distance between the estimated and true signal subspace (lower is better; shading shows $\pm 1$ std over 20 trials). CP fails at $\nu \approx 0.15$ while CA remains robust until $\nu \approx 0.75$. \figright~Theoretical separation ratios $\Delta_{\mathrm{CA}}$ and $\Delta_{\mathrm{CP}}$ (log scale). The dashed line marks $\Delta = 1$; recovery succeeds above this threshold.}\label{fig:linear_E4}
    \end{minipage}
\end{figure}

\begin{figure}[h]
  \centering
  \includegraphics[width=\linewidth]{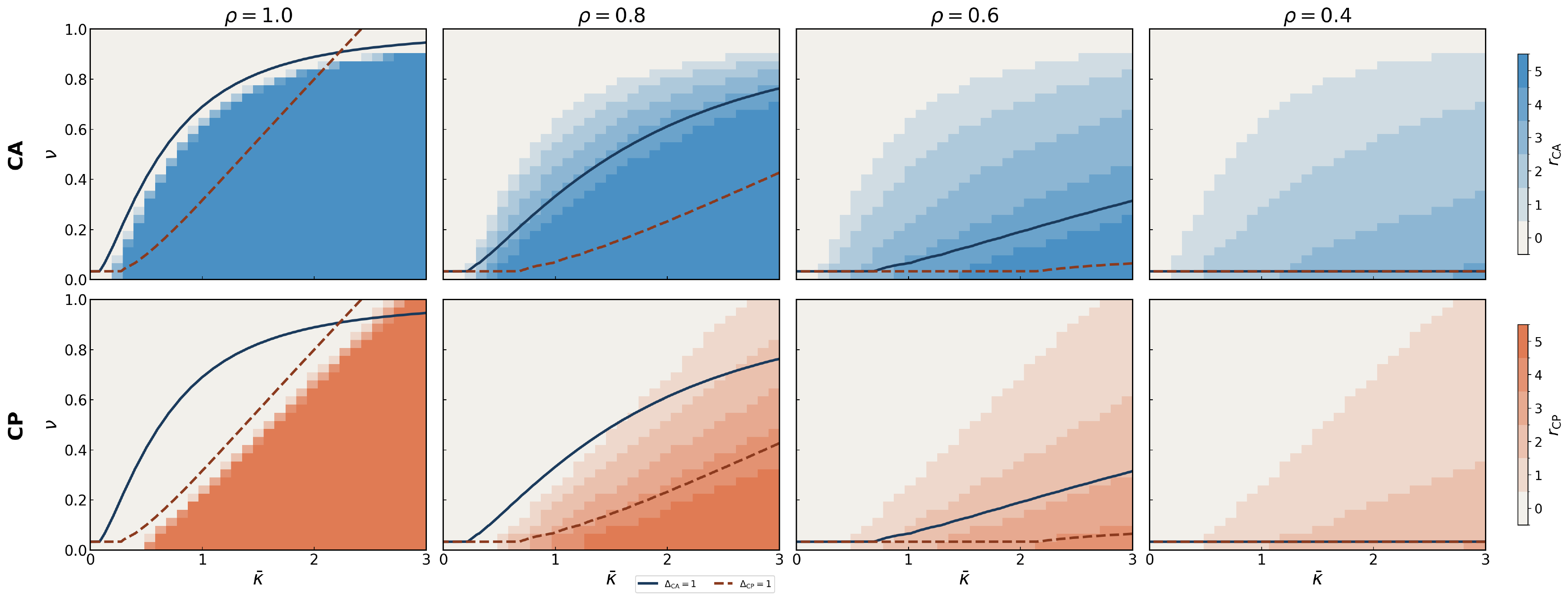}
  \caption{\textbf{Partial recovery under heterogeneous signal spectra.} 
  Empirical recovery count $r$ (number of signal directions with squared 
  projection $\geq 0.8$ onto the top-$k$ recovered subspace) across the 
  $(\bar{\kappa},\,\nu)$ plane, for signal spread 
  $\rho \in \{1.0,\,0.8,\,0.6,\,0.4\}$ (columns, homogeneous $\to$ 
  heterogeneous). \figtop: CA. \figbottom: CP. Signal strengths follow 
  a geometric decay $\kappa_i = \bar{\kappa}\,\rho^{\,i-1}$ with $k=5$, 
  $d=20$; noise parameters $\gamma_x = 1$, $\gamma_y = 0.05$, 
  $\tilde{\gamma}_y = 5$. Overlaid curves: $\DCA = 1$ (solid) and 
  $\DCP = 1$ (dashed) from the homogeneous theory of 
  \figref{fig:phase_diagram}. At $\rho = 1$ the transition is sharp 
  ($r \in \{0, k\}$) and the $\Delta = 1$ contours align with the empirical 
  boundary, reproducing \figref{fig:phase_diagram}. As $\rho$ decreases, 
  intermediate counts $r \in \{1,\dots,k-1\}$ fill a widening band in which 
  stronger signal directions are recovered first; the four-region phase 
  diagram smears into a graded continuum. Averaged over 
  10 seeds with $n = 5{,}000$.}
  \label{fig:partial_recovery}
\end{figure}

\begin{figure}[h]
    \centering
    \includegraphics[width=0.5\linewidth]{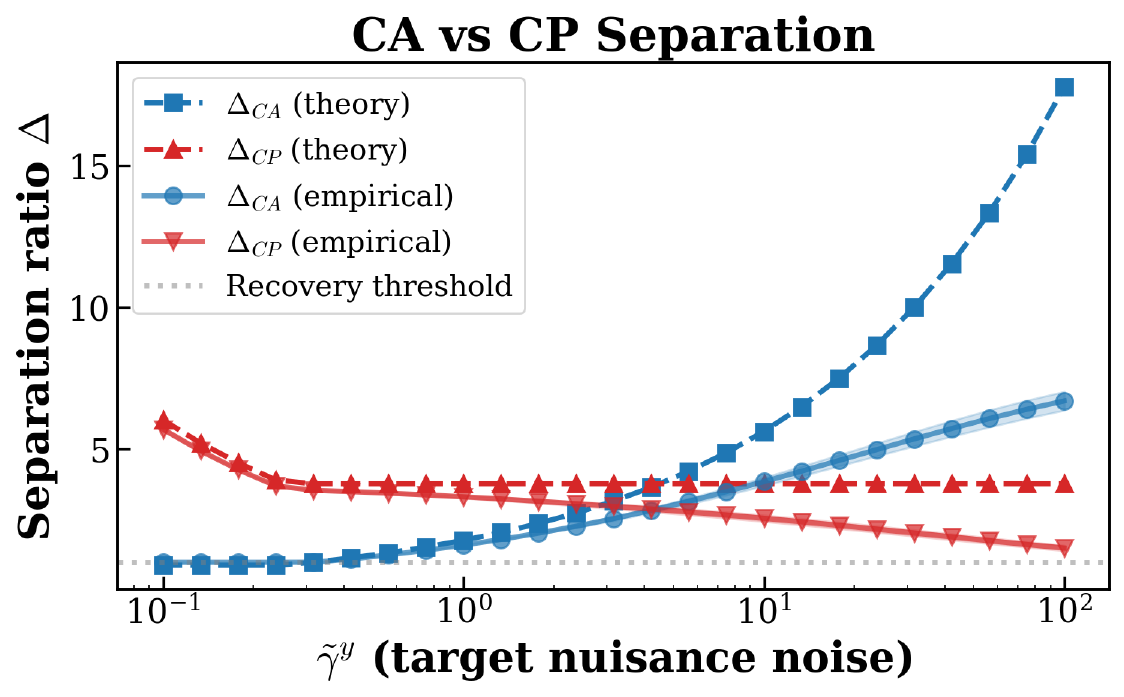}
    \caption{Separation ratios $\Delta_{\mathrm{CA}}$ and $\Delta_{\mathrm{CP}}$ as a function of
    target nuisance variance $\tilde{\gamma}^y$, validating~\Cref{prop:separation}.
    Theory curves (dashed) are computed from the closed-form expressions in
    \Cref{thm:closed_form_ca,thm:closed_form_cp}; empirical curves (solid) are estimated from finite-sample
    covariances averaged over 20 random rotations.
    As $\tilde{\gamma}^y$ grows, $\Delta_{\mathrm{CA}}$ increases unboundedly —
    CA's symmetric whitening suppresses high-variance nuisance on both sides —
    while $\Delta_{\mathrm{CP}}$ remains approximately constant, since 
    $\xi_j = \eta_j/\sqrt{\tilde\gamma^x_j}$ does not depend on target nuisance 
    variance: source-side whitening operates only on the source modality.
    Theory and empirical curves are in close agreement throughout, confirming the
    accuracy of the closed-form predictions at finite sample sizes.}
    \label{fig:ca_cp_separation}
\end{figure}

\begin{figure}[h]
    \centering
    \includegraphics[width=0.6\linewidth]{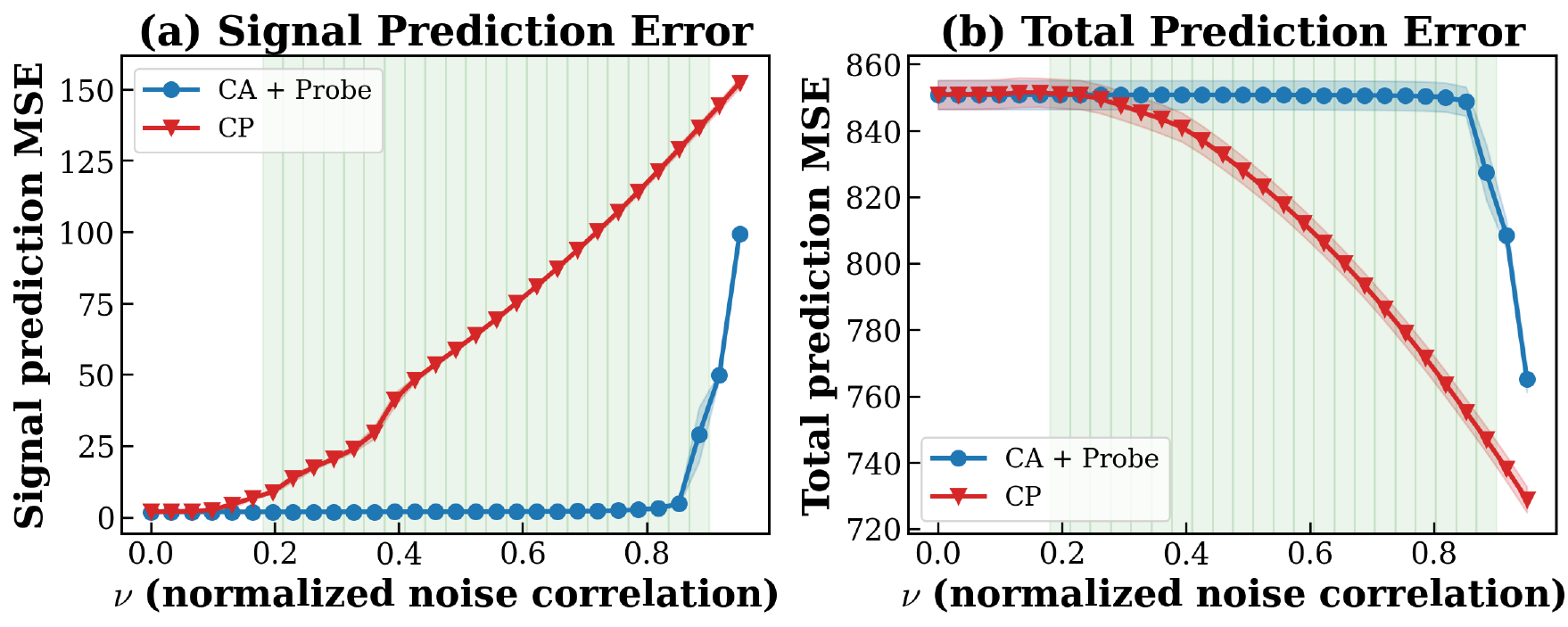}
    \caption{The variance trap: signal recovery vs.\ reconstruction quality for CA+Probe
    and direct CP, using the same parameter regime as~\Cref{fig:linear_E4}
    ($\kappa = (3.0, 2.0, 1.5)$, $\tilde{\gamma}^y = 50.0$, $\tilde{\gamma}^x = 1.0$).
    Green shading marks the regime where $\Delta_{\mathrm{CA}} > 1 > \Delta_{\mathrm{CP}}$.
    \textbf{(a) Signal prediction MSE}: in the green region, CA+Probe achieves
    near-zero signal MSE while CP's signal error climbs steeply, confirming that CP
    encodes the wrong subspace.
    \textbf{(b) Total prediction MSE}: CP achieves \emph{lower} total MSE than CA+Probe
    across the same region, because it successfully reconstructs the high-variance nuisance
    components — a task at which CA+Probe, having discarded nuisance directions, cannot
    compete.
    Together, the two panels illustrate the core danger of using reconstruction error as a
    proxy for signal recovery: CP can achieve lower loss than CA while failing to recover
    the signal, precisely because the MSE objective does not distinguish signal from nuisance.}
    \label{fig:ca_probe_vs_cp}
\end{figure}

\begin{figure}[h]
    \centering
    \begin{minipage}[t]{0.48\linewidth}
        \centering
        \includegraphics[width=\linewidth]{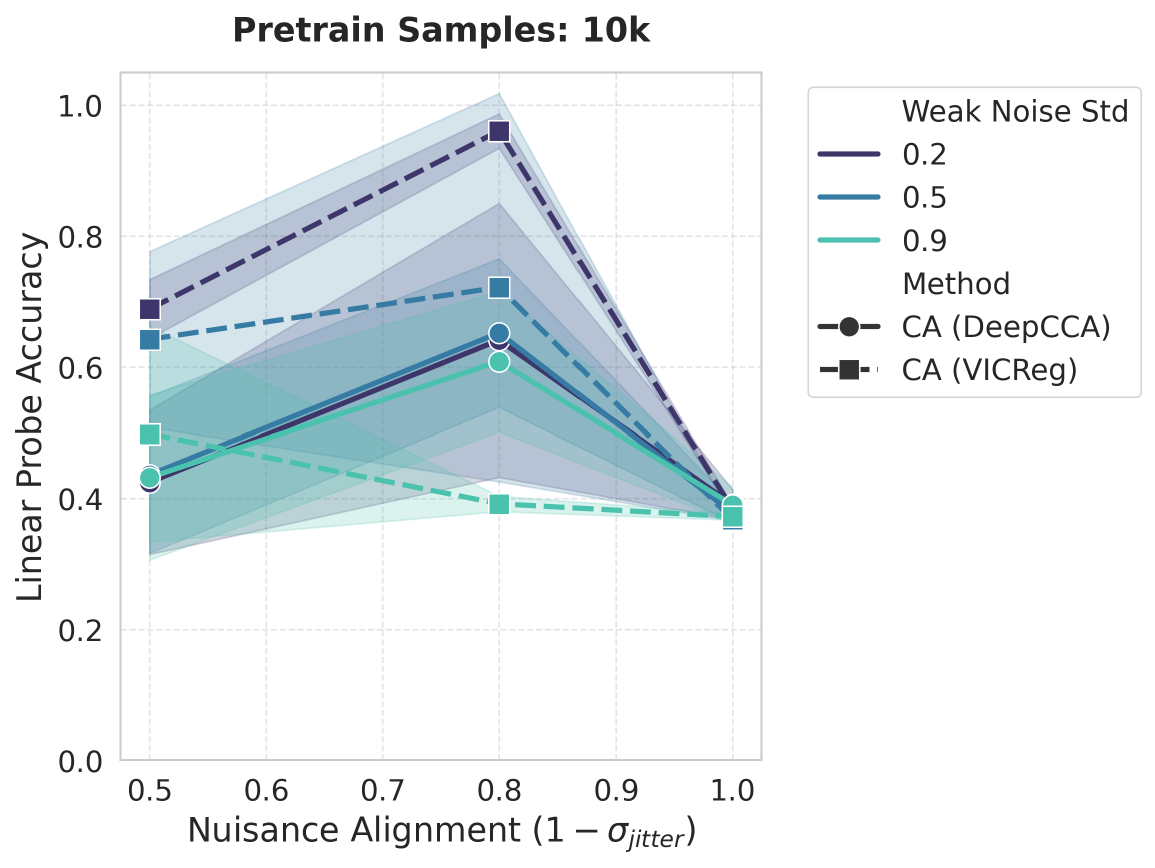}
        \par\smallskip
        \small (a) dSprites
    \end{minipage}
    \hfill
    \begin{minipage}[t]{0.48\linewidth}
        \centering
        \includegraphics[width=\linewidth]{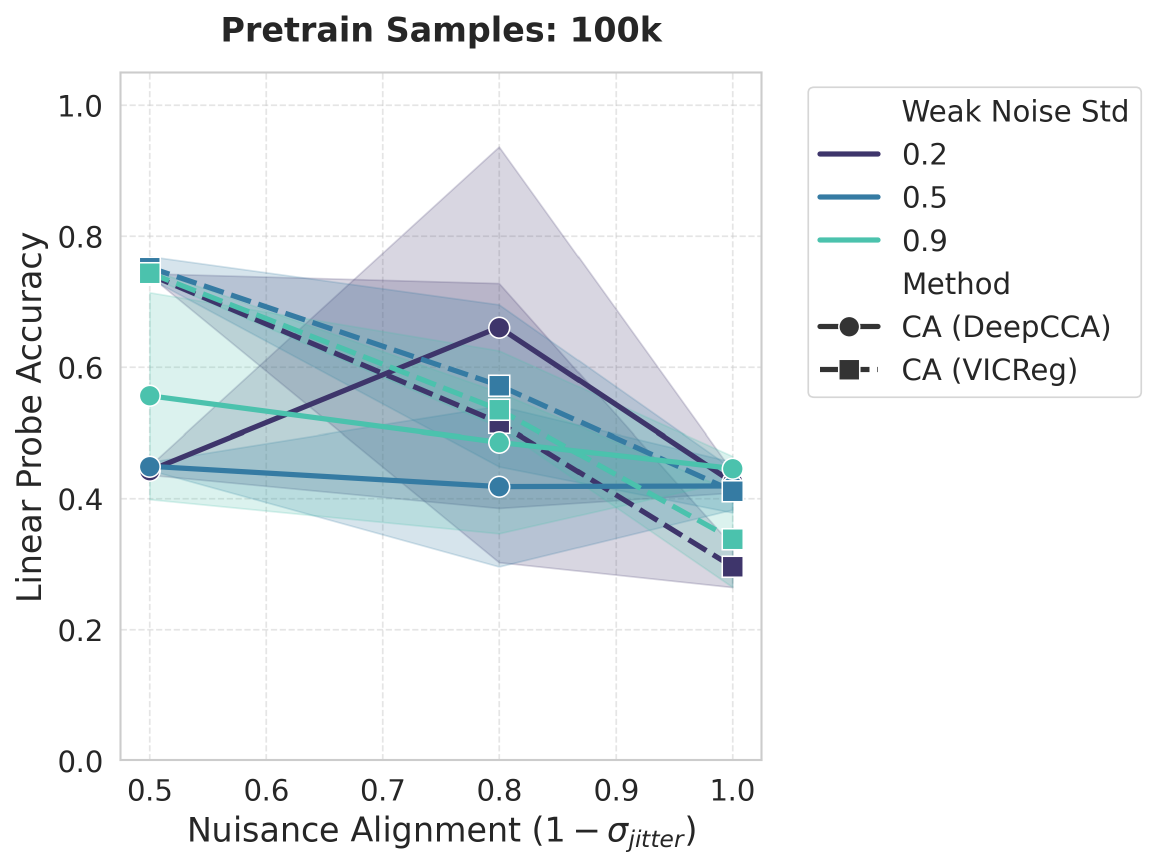}
        \par\smallskip
        \small (b) Shape3D
    \end{minipage}
    \caption{Comparison between VICReg and DeepCCA for dSprites and Shape3D experiments}
    \label{fig:vicreg_vs_deepcca}
\end{figure}

\begin{figure}[h]
    \centering
    \begin{minipage}[t]{0.48\linewidth}
        \centering
        \includegraphics[width=\linewidth]{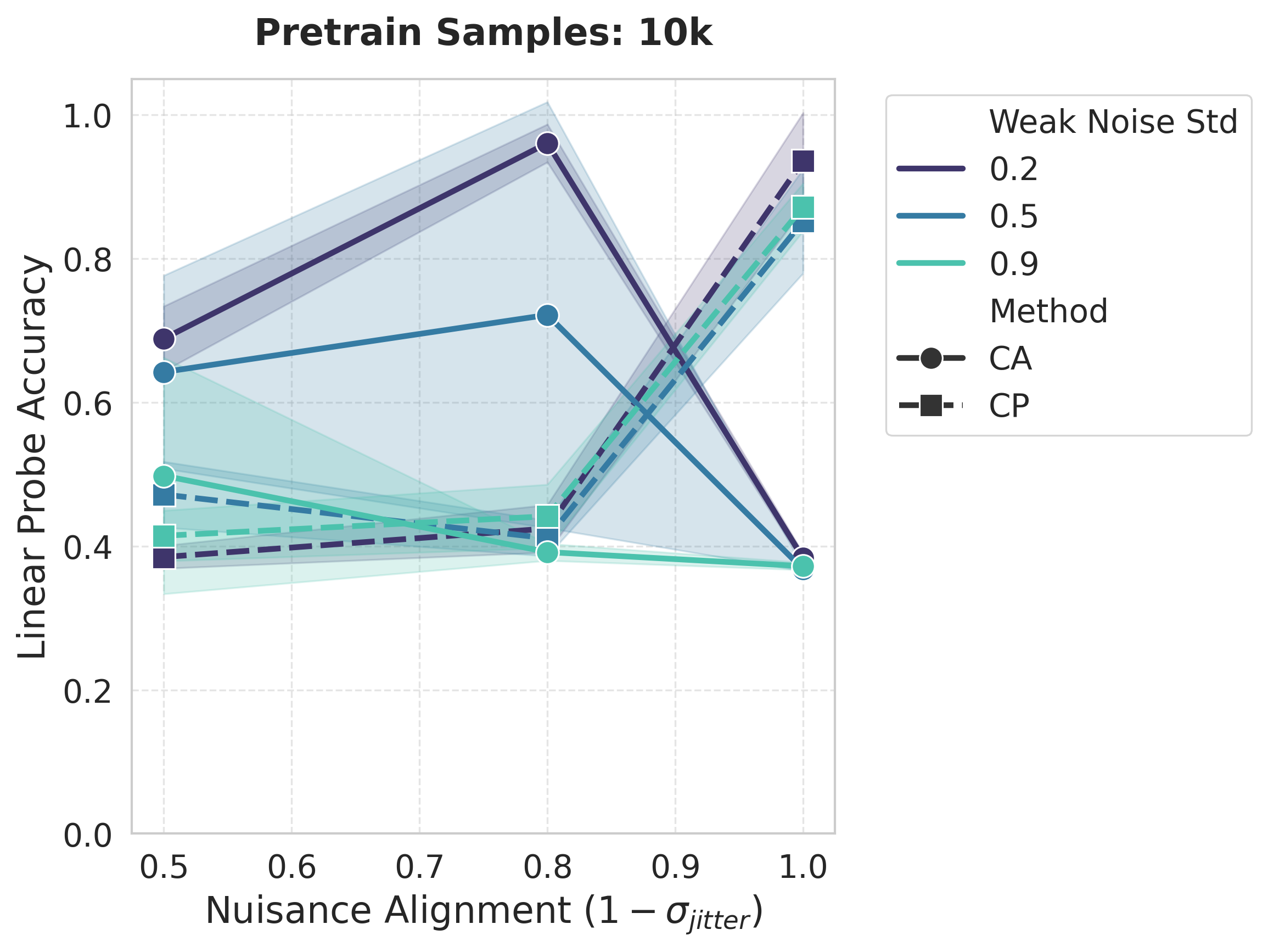}
    \end{minipage}
    \hfill
    \begin{minipage}[t]{0.48\linewidth}
        \centering
        \includegraphics[width=\linewidth]{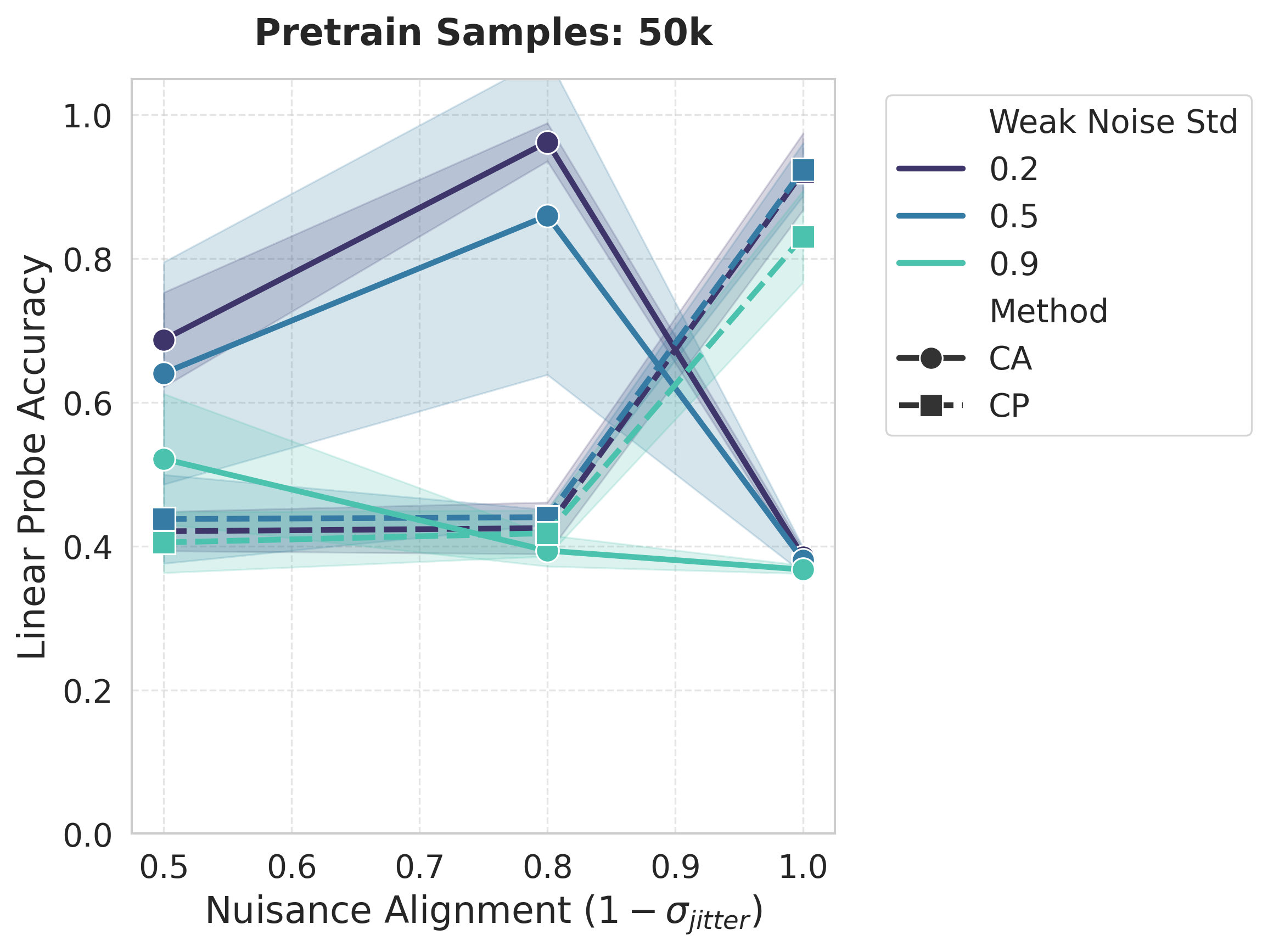}
    \end{minipage}
    \caption{Stereo-dSprites accuracy vs.\ nuisance alignment at 10k \figleft~and 50k \figright~pretraining samples. The CA--CP crossover from Figure~\ref{fig:stereo_performance} persists across dataset scales.}
    \label{fig:dsprites_samples}
\end{figure}

\begin{figure}[t]
  \centering
  \includegraphics[width=\linewidth]{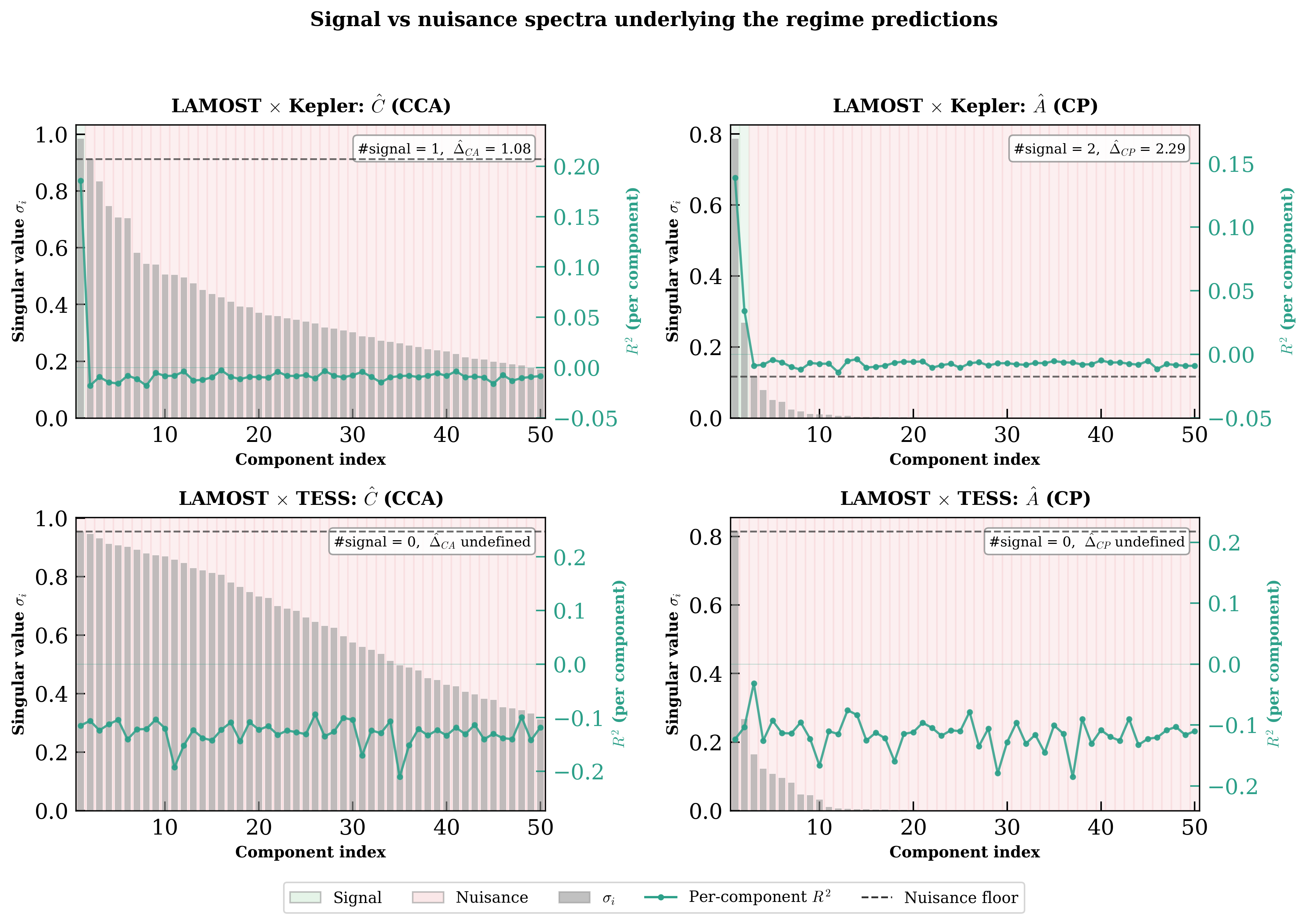}
 \caption{\textbf{Signal--nuisance decomposition underlying the regime
  predictions of \Cref{tab:astro_results}.} Each panel shows the sorted
  singular values (gray bars, left axis) and per-component $R^2$ against
  $\log g$ (teal line, right axis) used by \Cref{alg:phase_estimation} to
  classify components as signal (green shading) or nuisance (red shading).
  Dashed horizontal line: nuisance floor $\max_j \hat\nu_j$ (CCA panels) or
  $\max_j \hat\xi_j$ ($\Ab$-SVD panels); $\hat\Delta > 1$ iff every classified signal singular value exceeds the 
  nuisance floor. \figtop: LAMOST $\times$ Kepler --- both
  decompositions have signal components above the nuisance floor
  ($\hat\Delta_{\mathrm{CA}} = 1.08$, $\hat\Delta_{\mathrm{CP}} = 2.29$; Both regime). \figbottom: LAMOST
  $\times$ TESS --- no CCA component predicts $\log g$ above noise
  ($R^2 \approx 0$ across all components, zero signal detected); the $\mathbf{A}$-SVD has one dominant singular value, but no component predicts $\log g$,
so it too is classified as nuisance. No signal direction clears the floor in either spectrum ($\hat r_{\mathrm{CA}} = \hat r_{\mathrm{CP}} = 0$), the complete-failure form of the Neither regime. The contrast between the two
  rows --- same LAMOST encoder, same protocol, different photometric
  instrument --- shows that instrument quality determines the
  signal--nuisance separation and hence the regime placement.}
  \label{fig:phase_diagnostics}
\end{figure}

\clearpage
\section*{NeurIPS Paper Checklist}

\begin{enumerate}

\item {\bf Claims}
    \item[] Question: Do the main claims made in the abstract and introduction accurately reflect the paper's contributions and scope?
    \item[] Answer: \answerYes{} 
    \item[] Justification: The four contributions stated in the introduction:
    closed-form solutions, separation ratios with phase diagrams, data-driven estimation, and experimental validation are delivered in
    \Cref{sec:solutions_and_bottlenecks}, and \Cref{sec:experiments}. Limitations of the linear analysis for single-stage pipelines are discussed in \Cref{sec:regime_prediction}.
    \item[] Guidelines:
    \begin{itemize}
        \item The answer \answerNA{} means that the abstract and introduction do not include the claims made in the paper.
        \item The abstract and/or introduction should clearly state the claims made, including the contributions made in the paper and important assumptions and limitations. A \answerNo{} or \answerNA{} answer to this question will not be perceived well by the reviewers. 
        \item The claims made should match theoretical and experimental results, and reflect how much the results can be expected to generalize to other settings. 
        \item It is fine to include aspirational goals as motivation as long as it is clear that these goals are not attained by the paper. 
    \end{itemize}

\item {\bf Limitations}
    \item[] Question: Does the paper discuss the limitations of the work performed by the authors?
    \item[] Answer: \answerYes{} 
    \item[] Justification:
    Limitations are discussed where they arise: applicability of the linear analysis to nonlinear regimes (\Cref{sec:experiments}), the scope of the diagnostic to two-stage
    pipelines (\Cref{sec:regime_prediction}), and the open Neither regime
    (\Cref{sec:conclusion}).

    \item[] Guidelines:
    \begin{itemize}
        \item The answer \answerNA{} means that the paper has no limitation while the answer \answerNo{} means that the paper has limitations, but those are not discussed in the paper. 
        \item The authors are encouraged to create a separate ``Limitations'' section in their paper.
        \item The paper should point out any strong assumptions and how robust the results are to violations of these assumptions (e.g., independence assumptions, noiseless settings, model well-specification, asymptotic approximations only holding locally). The authors should reflect on how these assumptions might be violated in practice and what the implications would be.
        \item The authors should reflect on the scope of the claims made, e.g., if the approach was only tested on a few datasets or with a few runs. In general, empirical results often depend on implicit assumptions, which should be articulated.
        \item The authors should reflect on the factors that influence the performance of the approach. For example, a facial recognition algorithm may perform poorly when image resolution is low or images are taken in low lighting. Or a speech-to-text system might not be used reliably to provide closed captions for online lectures because it fails to handle technical jargon.
        \item The authors should discuss the computational efficiency of the proposed algorithms and how they scale with dataset size.
        \item If applicable, the authors should discuss possible limitations of their approach to address problems of privacy and fairness.
        \item While the authors might fear that complete honesty about limitations might be used by reviewers as grounds for rejection, a worse outcome might be that reviewers discover limitations that aren't acknowledged in the paper. The authors should use their best judgment and recognize that individual actions in favor of transparency play an important role in developing norms that preserve the integrity of the community. Reviewers will be specifically instructed to not penalize honesty concerning limitations.
    \end{itemize}

\item {\bf Theory assumptions and proofs}
    \item[] Question: For each theoretical result, does the paper provide the full set of assumptions and a complete (and correct) proof?
    \item[] Answer: \answerYes{} 
    \item[] Justification: All the assumptions and proofs are provided in ~\Cref{app:closed_form}. Closed form solutions are given in ~\Cref{app:closed_form_statement}, proofs are given in ~\Cref{app:proof_ca} and ~\Cref{app:proof_cp}. Proofs to ~\Cref{prop:separation} and ~\Cref{prop:partial_recovery} are given in ~\Cref{app:separation_proof} and ~\Cref{app:partial_recovery_proof}.
    \item[] Guidelines:
    \begin{itemize}
        \item The answer \answerNA{} means that the paper does not include theoretical results. 
        \item All the theorems, formulas, and proofs in the paper should be numbered and cross-referenced.
        \item All assumptions should be clearly stated or referenced in the statement of any theorems.
        \item The proofs can either appear in the main paper or the supplemental material, but if they appear in the supplemental material, the authors are encouraged to provide a short proof sketch to provide intuition. 
        \item Inversely, any informal proof provided in the core of the paper should be complemented by formal proofs provided in appendix or supplemental material.
        \item Theorems and Lemmas that the proof relies upon should be properly referenced. 
    \end{itemize}

    \item {\bf Experimental result reproducibility}
    \item[] Question: Does the paper fully disclose all the information needed to reproduce the main experimental results of the paper to the extent that it affects the main claims and/or conclusions of the paper (regardless of whether the code and data are provided or not)?
    \item[] Answer: \answerYes{} 
    \item[] Justification: ~\Cref{sec:implementation_details} provides architectures,
    training hyperparameters, optimization details, evaluation protocols, and
    sample sizes for all five experiments. ~\Cref{alg:phase_estimation}
    is stated in full. Code will be released at the anonymous repository linked
    on page~1.
    \item[] Guidelines:
    \begin{itemize}
        \item The answer \answerNA{} means that the paper does not include experiments.
        \item If the paper includes experiments, a \answerNo{} answer to this question will not be perceived well by the reviewers: Making the paper reproducible is important, regardless of whether the code and data are provided or not.
        \item If the contribution is a dataset and\slash or model, the authors should describe the steps taken to make their results reproducible or verifiable. 
        \item Depending on the contribution, reproducibility can be accomplished in various ways. For example, if the contribution is a novel architecture, describing the architecture fully might suffice, or if the contribution is a specific model and empirical evaluation, it may be necessary to either make it possible for others to replicate the model with the same dataset, or provide access to the model. In general. releasing code and data is often one good way to accomplish this, but reproducibility can also be provided via detailed instructions for how to replicate the results, access to a hosted model (e.g., in the case of a large language model), releasing of a model checkpoint, or other means that are appropriate to the research performed.
        \item While NeurIPS does not require releasing code, the conference does require all submissions to provide some reasonable avenue for reproducibility, which may depend on the nature of the contribution. For example
        \begin{enumerate}
            \item If the contribution is primarily a new algorithm, the paper should make it clear how to reproduce that algorithm.
            \item If the contribution is primarily a new model architecture, the paper should describe the architecture clearly and fully.
            \item If the contribution is a new model (e.g., a large language model), then there should either be a way to access this model for reproducing the results or a way to reproduce the model (e.g., with an open-source dataset or instructions for how to construct the dataset).
            \item We recognize that reproducibility may be tricky in some cases, in which case authors are welcome to describe the particular way they provide for reproducibility. In the case of closed-source models, it may be that access to the model is limited in some way (e.g., to registered users), but it should be possible for other researchers to have some path to reproducing or verifying the results.
        \end{enumerate}
    \end{itemize}

\item {\bf Open access to data and code}
    \item[] Question: Does the paper provide open access to the data and code, with sufficient instructions to faithfully reproduce the main experimental results, as described in supplemental material?
    \item[] Answer: \answerYes{}{} 
    \item[] Justification: An anonymized repository with code, configurations, and  per-experiment READMEs is linked on page 1. All synthetic data is generated 
    by included scripts; public datasets (MS-COCO, LAMOST, Kepler, TESS) are 
    accessed via documented procedures. The repository will be open-sourced 
    upon publication.
    \item[] Guidelines:
    \begin{itemize}
        \item The answer \answerNA{} means that paper does not include experiments requiring code.
        \item Please see the NeurIPS code and data submission guidelines (\url{https://neurips.cc/public/guides/CodeSubmissionPolicy}) for more details.
        \item While we encourage the release of code and data, we understand that this might not be possible, so \answerNo{} is an acceptable answer. Papers cannot be rejected simply for not including code, unless this is central to the contribution (e.g., for a new open-source benchmark).
        \item The instructions should contain the exact command and environment needed to run to reproduce the results. See the NeurIPS code and data submission guidelines (\url{https://neurips.cc/public/guides/CodeSubmissionPolicy}) for more details.
        \item The authors should provide instructions on data access and preparation, including how to access the raw data, preprocessed data, intermediate data, and generated data, etc.
        \item The authors should provide scripts to reproduce all experimental results for the new proposed method and baselines. If only a subset of experiments are reproducible, they should state which ones are omitted from the script and why.
        \item At submission time, to preserve anonymity, the authors should release anonymized versions (if applicable).
        \item Providing as much information as possible in supplemental material (appended to the paper) is recommended, but including URLs to data and code is permitted.
    \end{itemize}

\item {\bf Experimental setting/details}
    \item[] Question: Does the paper specify all the training and test details (e.g., data splits, hyperparameters, how they were chosen, type of optimizer) necessary to understand the results?
    \item[] Answer: \answerYes{}{} 
    \item[] Justification:
    Section B provides architectures, optimizers, learning rates, 
    batch sizes, training durations, and evaluation protocols for 
    all five experiments.
    \item[] Guidelines:
    \begin{itemize}
        \item The answer \answerNA{} means that the paper does not include experiments.
        \item The experimental setting should be presented in the core of the paper to a level of detail that is necessary to appreciate the results and make sense of them.
        \item The full details can be provided either with the code, in appendix, or as supplemental material.
    \end{itemize}

\item {\bf Experiment statistical significance}
    \item[] Question: Does the paper report error bars suitably and correctly defined or other appropriate information about the statistical significance of the experiments?
    \item[] Answer: \answerYes{}{} 
    \item[] Justification: Linear experiments report std over 20 rotations; stereo-vision experiments show shaded bands across multiple seeds. Astrophysical experiments report mean $\pm$ std over 5 seeds; The MS-COCO experiment reports a single run per condition due to computational cost.
    \item[] Guidelines:
    \begin{itemize}
        \item The answer \answerNA{} means that the paper does not include experiments.
        \item The authors should answer \answerYes{} if the results are accompanied by error bars, confidence intervals, or statistical significance tests, at least for the experiments that support the main claims of the paper.
        \item The factors of variability that the error bars are capturing should be clearly stated (for example, train/test split, initialization, random drawing of some parameter, or overall run with given experimental conditions).
        \item The method for calculating the error bars should be explained (closed form formula, call to a library function, bootstrap, etc.)
        \item The assumptions made should be given (e.g., Normally distributed errors).
        \item It should be clear whether the error bar is the standard deviation or the standard error of the mean.
        \item It is OK to report 1-sigma error bars, but one should state it. The authors should preferably report a 2-sigma error bar than state that they have a 96\% CI, if the hypothesis of Normality of errors is not verified.
        \item For asymmetric distributions, the authors should be careful not to show in tables or figures symmetric error bars that would yield results that are out of range (e.g., negative error rates).
        \item If error bars are reported in tables or plots, the authors should explain in the text how they were calculated and reference the corresponding figures or tables in the text.
    \end{itemize}

\item {\bf Experiments compute resources}
    \item[] Question: For each experiment, does the paper provide sufficient information on the computer resources (type of compute workers, memory, time of execution) needed to reproduce the experiments?
    \item[] Answer: \answerYes{} 
    \item[] Justification: 
    For each experiment, we specify the GPU types, number of GPUs, and training times in ~\Cref{sec:implementation_details}. Total compute for the project, including preliminary experiments not reported in the paper, is approximately $1000$ GPU-hours.
    \item[] Guidelines:
    \begin{itemize}
        \item The answer \answerNA{} means that the paper does not include experiments.
        \item The paper should indicate the type of compute workers CPU or GPU, internal cluster, or cloud provider, including relevant memory and storage.
        \item The paper should provide the amount of compute required for each of the individual experimental runs as well as estimate the total compute. 
        \item The paper should disclose whether the full research project required more compute than the experiments reported in the paper (e.g., preliminary or failed experiments that didn't make it into the paper). 
    \end{itemize}
    
\item {\bf Code of ethics}
    \item[] Question: Does the research conducted in the paper conform, in every respect, with the NeurIPS Code of Ethics \url{https://neurips.cc/public/EthicsGuidelines}?
    \item[] Answer: \answerYes{} 
    \item[] Justification: The research involves no human subjects, no scraped data, and no models with foreseeable misuse risks. All datasets used are public and used within their licensing terms.
    \item[] Guidelines:
    \begin{itemize}
        \item The answer \answerNA{} means that the authors have not reviewed the NeurIPS Code of Ethics.
        \item If the authors answer \answerNo, they should explain the special circumstances that require a deviation from the Code of Ethics.
        \item The authors should make sure to preserve anonymity (e.g., if there is a special consideration due to laws or regulations in their jurisdiction).
    \end{itemize}

\item {\bf Broader impacts}
    \item[] Question: Does the paper discuss both potential positive societal impacts and negative societal impacts of the work performed?
    \item[] Answer: \answerNA{} 
    \item[] Justification: This paper is a theoretical analysis of multimodal 
    representation learning objectives. It introduces no new datasets, models, 
    or applications with direct societal consequences.
    \item[] Guidelines:
    \begin{itemize}
        \item The answer \answerNA{} means that there is no societal impact of the work performed.
        \item If the authors answer \answerNA{} or \answerNo, they should explain why their work has no societal impact or why the paper does not address societal impact.
        \item Examples of negative societal impacts include potential malicious or unintended uses (e.g., disinformation, generating fake profiles, surveillance), fairness considerations (e.g., deployment of technologies that could make decisions that unfairly impact specific groups), privacy considerations, and security considerations.
        \item The conference expects that many papers will be foundational research and not tied to particular applications, let alone deployments. However, if there is a direct path to any negative applications, the authors should point it out. For example, it is legitimate to point out that an improvement in the quality of generative models could be used to generate Deepfakes for disinformation. On the other hand, it is not needed to point out that a generic algorithm for optimizing neural networks could enable people to train models that generate Deepfakes faster.
        \item The authors should consider possible harms that could arise when the technology is being used as intended and functioning correctly, harms that could arise when the technology is being used as intended but gives incorrect results, and harms following from (intentional or unintentional) misuse of the technology.
        \item If there are negative societal impacts, the authors could also discuss possible mitigation strategies (e.g., gated release of models, providing defenses in addition to attacks, mechanisms for monitoring misuse, mechanisms to monitor how a system learns from feedback over time, improving the efficiency and accessibility of ML).
    \end{itemize}
    
\item {\bf Safeguards}
    \item[] Question: Does the paper describe safeguards that have been put in place for responsible release of data or models that have a high risk for misuse (e.g., pre-trained language models, image generators, or scraped datasets)?
    \item[] Answer: \answerNA{} 
    \item[] Justification: The paper releases no pretrained models, datasets, or generative systems with high misuse risk.
    \item[] Guidelines:
    \begin{itemize}
        \item The answer \answerNA{} means that the paper poses no such risks.
        \item Released models that have a high risk for misuse or dual-use should be released with necessary safeguards to allow for controlled use of the model, for example by requiring that users adhere to usage guidelines or restrictions to access the model or implementing safety filters. 
        \item Datasets that have been scraped from the Internet could pose safety risks. The authors should describe how they avoided releasing unsafe images.
        \item We recognize that providing effective safeguards is challenging, and many papers do not require this, but we encourage authors to take this into account and make a best faith effort.
    \end{itemize}

\item {\bf Licenses for existing assets}
    \item[] Question: Are the creators or original owners of assets (e.g., code, data, models), used in the paper, properly credited and are the license and terms of use explicitly mentioned and properly respected?
    \item[] Answer: \answerYes{} 
    \item[] Justification: All datasets used (MS-COCO, dSprites, 3DShapes, LAMOST, 
    Kepler, TESS) are publicly available and properly cited.
    \item[] Guidelines:
    \begin{itemize}
        \item The answer \answerNA{} means that the paper does not use existing assets.
        \item The authors should cite the original paper that produced the code package or dataset.
        \item The authors should state which version of the asset is used and, if possible, include a URL.
        \item The name of the license (e.g., CC-BY 4.0) should be included for each asset.
        \item For scraped data from a particular source (e.g., website), the copyright and terms of service of that source should be provided.
        \item If assets are released, the license, copyright information, and terms of use in the package should be provided. For popular datasets, \url{paperswithcode.com/datasets} has curated licenses for some datasets. Their licensing guide can help determine the license of a dataset.
        \item For existing datasets that are re-packaged, both the original license and the license of the derived asset (if it has changed) should be provided.
        \item If this information is not available online, the authors are encouraged to reach out to the asset's creators.
    \end{itemize}

\item {\bf New assets}
    \item[] Question: Are new assets introduced in the paper well documented and is the documentation provided alongside the assets?
    \item[] Answer: \answerNA{} 
    \item[] Justification: The paper introduces no new datasets or pretrained models. 
    \item[] Guidelines:
    \begin{itemize}
        \item The answer \answerNA{} means that the paper does not release new assets.
        \item Researchers should communicate the details of the dataset\slash code\slash model as part of their submissions via structured templates. This includes details about training, license, limitations, etc. 
        \item The paper should discuss whether and how consent was obtained from people whose asset is used.
        \item At submission time, remember to anonymize your assets (if applicable). You can either create an anonymized URL or include an anonymized zip file.
    \end{itemize}

\item {\bf Crowdsourcing and research with human subjects}
    \item[] Question: For crowdsourcing experiments and research with human subjects, does the paper include the full text of instructions given to participants and screenshots, if applicable, as well as details about compensation (if any)? 
    \item[] Answer: \answerNA{} 
    \item[] Justification: No crowdsourcing or human subjects research.
    \item[] Guidelines:
    \begin{itemize}
        \item The answer \answerNA{} means that the paper does not involve crowdsourcing nor research with human subjects.
        \item Including this information in the supplemental material is fine, but if the main contribution of the paper involves human subjects, then as much detail as possible should be included in the main paper. 
        \item According to the NeurIPS Code of Ethics, workers involved in data collection, curation, or other labor should be paid at least the minimum wage in the country of the data collector. 
    \end{itemize}

\item {\bf Institutional review board (IRB) approvals or equivalent for research with human subjects}
    \item[] Question: Does the paper describe potential risks incurred by study participants, whether such risks were disclosed to the subjects, and whether Institutional Review Board (IRB) approvals (or an equivalent approval/review based on the requirements of your country or institution) were obtained?
    \item[] Answer: \answerNA{} 
    \item[] Justification: No human subjects research.
    \item[] Guidelines:
    \begin{itemize}
        \item The answer \answerNA{} means that the paper does not involve crowdsourcing nor research with human subjects.
        \item Depending on the country in which research is conducted, IRB approval (or equivalent) may be required for any human subjects research. If you obtained IRB approval, you should clearly state this in the paper. 
        \item We recognize that the procedures for this may vary significantly between institutions and locations, and we expect authors to adhere to the NeurIPS Code of Ethics and the guidelines for their institution. 
        \item For initial submissions, do not include any information that would break anonymity (if applicable), such as the institution conducting the review.
    \end{itemize}

\item {\bf Declaration of LLM usage}
    \item[] Question: Does the paper describe the usage of LLMs if it is an important, original, or non-standard component of the core methods in this research? Note that if the LLM is used only for writing, editing, or formatting purposes and does \emph{not} impact the core methodology, scientific rigor, or originality of the research, declaration is not required.
    \item[] Answer: \answerNA{} 
    \item[] Justification:
    LLMs were used only for editing and writing assistance (code and text).
    \item[] Guidelines:
    \begin{itemize}
        \item The answer \answerNA{} means that the core method development in this research does not involve LLMs as any important, original, or non-standard components.
        \item Please refer to our LLM policy in the NeurIPS handbook for what should or should not be described.
    \end{itemize}

\end{enumerate}
\end{document}